\documentclass[final]{cvpr}

\usepackage{times}
\usepackage{epsfig}
\usepackage{graphicx}
\usepackage{amsmath}
\usepackage{amssymb}
\usepackage{float} 
\usepackage{subfig}
\usepackage[hyphens]{url}  
\usepackage{caption}
\usepackage{multirow}
\usepackage[ruled,vlined]{algorithm2e}
\usepackage{amsthm}
\usepackage{amsfonts}
\usepackage{mathrsfs}
\usepackage{calligra}
\usepackage{bbm}
\usepackage[switch]{lineno} 
\usepackage{enumitem}
\usepackage{booktabs}

\usepackage{array}
\usepackage{tabularx}
\usepackage{multirow}
\usepackage{mathtools}
\usepackage{optidef}
\usepackage{xcolor}
\newtheorem{prop}{Proposition}
\newcommand{\norm}[1]{\left\lVert#1\right\rVert}
\DeclareMathOperator*{\argmax}{arg\,max}
\DeclareMathOperator*{\argmin}{arg\,min}
\usepackage{siunitx}
\usepackage{makecell, booktabs, caption}
\usepackage[pagebackref=true,breaklinks=true,colorlinks,bookmarks=false]{hyperref}

\def\cvprPaperID{****} 
\def\confYear{CVPR 2021}

\begin{document}

\title{Generalization on Unseen Domains via Inference-time Label-Preserving Target Projections}

\author{
Prashant Pandey\textsuperscript{\rm 1},
Mrigank Raman\thanks{Equal contribution} \textsuperscript{\rm 1},
Sumanth Varambally\footnotemark[1] \textsuperscript{\rm 1},
Prathosh AP\textsuperscript{\rm 1}\\
\textsuperscript{\rm 1}IIT Delhi\\
\{\tt\small bsz178495, mt1170736, mt6170855, prathoshap\}@iitd.ac.in
}

\maketitle


\maketitle

\begin{abstract}
 Generalization of machine learning models trained on a set of source domains on unseen target domains with different statistics, is a challenging problem. While many approaches have been proposed to solve this problem, they only utilize source data during training but do not take advantage of the fact that a single target example is available at the time of inference. Motivated by this, we propose a method that effectively uses the target sample during inference beyond mere classification. Our method has three components - (i) A label-preserving feature or metric transformation on source data such that the source samples are clustered in accordance with their class irrespective of their domain (ii) A generative model trained on the these features (iii) A label-preserving projection of the target point on the source-feature manifold during inference via solving an optimization problem on the input space of the generative model using the learned metric. Finally, the projected target is used in the classifier. Since the projected target feature comes from the source manifold and has the same label as the real target by design, the classifier is expected to perform better on it than the true target. We demonstrate that our method outperforms the state-of-the-art Domain Generalization methods on multiple datasets and tasks.
\end{abstract}
\vspace{-0.6cm}
\section{Introduction}
\vspace{-0.15cm}

Domain shift refers to the existence of significant divergence between the distributions of the training and the test data \cite{torralba2011unbiased}. This causes the machine learning models trained only on the training or the source data to perform poorly on the test or target data. A naive way of handling this problem is to fine-tune the model with new data which is often infeasible because of the difficulty in acquiring labelled data for every new target domain. The class of Domain Adaptation (DA) methods  \cite{tzeng2017adversarial, ganin2015unsupervised, hoffman2018cycada, sun2016return, long2016unsupervised, bousmalis2017unsupervised, murez2018image, panareda2017open, pandey2020skin} tackle this problem by utilizing the (unlabeled) target data to minimize the domain shift; however they cannot be used when unlabeled target data is unavailable.

Domain generalization (DG) \cite{muandet2013domain,li2017deeper,ghifary2015domain,li2018learning,balaji2018metareg,li2019episodic}, on the other hand, views the problem from the following perspective: how to make a model trained on single or multiple source domains generalize on completely unseen target domains. These methods do so via (i) learning feature representations that are invariant to the data domains using methods such as adversarial learning \cite{li2018domain, li2018deep}, (ii) simulating the domain shift while learning through meta-learning approaches \cite{li2018learning, balaji2018metareg}, and (iii) augmenting the source dataset with synthesized data from fictitious target domains \cite{volpi2018generalizing, zhou2020deep}. These methods have been shown to be effective in dealing with the problem of domain shift. 
However, most of the existing methods do not utilize the test sample from the target distribution available at the time of inference beyond mere classification. On the other hand, it is a common experience that when humans encounter an unseen object, they often relate it to a previously perceived similar object.

Motivated by this intuition, in this paper, we make the following contributions towards addressing the problem of DG: (a) Given samples from multiple source distributions, we propose to learn a source domain invariant representation that also preserves the class labels. (b) We propose to `project' the target samples to the manifold of the source-data features before classification through an inference-time label-preserving optimization procedure over the input of a generative model (learned during training) that maps an arbitrary distribution (Normal) to the source-feature manifold. (c) We demonstrate through extensive experimentation that our method achieves new state-of-the-art performance on standard DG tasks while also outperforming other methods in terms of robustness and data efficiency. 

\vspace{-0.35cm}
\section{Prior Work}
\vspace{-0.2cm}
\textbf{Meta-learning} : Meta-learning methods aim to improve model robustness against unseen domains by simulating domain shift during training. This is done by splitting the training set into a meta-train and meta-test set. \cite{li2018learning} provide a general framework for meta-learning-based DG, where model parameters are updated to minimize loss over the meta-train and meta-test domains in a coordinated manner. \cite{balaji2018metareg} propose a pre-trained regularizer network which is used to regularize the learning objective of a domain-independent task network. \cite{d2018domain} use a common feature extractor backbone network in conjunction with several domain-specific aggregation modules. An aggregation over these modules is performed during inference to predict the class label. \cite{li2019episodic} train separate feature extractors and classifiers on each of the source domains and minimize the loss on mismatched pairs of feature extractors and classifiers to improve model robustness. \cite{hu2019domain} utilise a probabilistic meta-learning model in which classifier parameters shared across domains are modeled as distributions. They also learn domain-invariant representations by optimizing a variational approximation to the information bottleneck. Since meta-learning methods are only trained on the simulated domain shifts, they might not always perform well on target domains that are not `covered' in the simulated shifts. \newline
\textbf{Data augmentation}: Augmenting the dataset with random transformations improves generalization \cite{hernandez2018further}. Commonly used augmentation techniques include rotation, flipping, random cropping, random colour distortions, amongst others.
\cite{shankar2018generalizing} use gradients from a domain classifier to perturb images. However, these perturbations might not be reflective of practically observed domain shift. \cite{zhou2020deep} aim to address this issue using an adversarial procedure to train a transformation network to produce an image translation that aims to generate novel domains while retaining class information. \cite{zhou2020learning} generate images from pseudo-novel domains with an optimal transport based formulation while preserving semantic information with cycle-consistency and classification losses. \cite{qiao2020learning} solve the problem of single-source DG by creating fictitious domains using Wasserstein Auto-Encoders in a meta-learning framework. While these generated domains differ significantly from the source domains, they potentially do not reflect practical domain differences.\newline
\textbf{Domain-invariant representations}: Another common pervasive theme in domain generalization literature is transforming the source data into a lower-dimensional `feature' space that is invariant to domains but retains the discriminative class information; these features are used for classification. \cite{ghifary2015domain} learn an auto-encoder to extract domain invariant features by reconstructing inter and cross domain images. \cite{li2018domain} use adversarial auto-encoders to align the representations from all the source domains to a Laplacian prior using adversarial learning procedure. \cite{dou2019domain} employ episodic training to simulate domain shift while minimizing a global class-alignment loss and local sample-clustering objective to cluster points class-wise. \cite{hu2019domain} learn a kernel function that minimizes mean domain discrepancy and intra-class scatter while maximizing mean class discrepancy and multi-domain between-class scatter. \cite{piratla2020efficient} propose a low-rank decomposition on the final classification layer to identifiably learn common and specific features across domains. \cite{seo2019learning} use domain-specific normalizations to learn representations that are domain-agnostic and semantically discriminative.

All of the above methods require domain labels, which might not always be viable. \cite{carlucci2019domain} aim to solve the problem of DG without domain labels by learning an auxiliary task of solving jigsaw puzzles. The idea is that features learned from such an auxiliary task will be invariant of the domains. \cite{matsuura2020domain} first assigns pseudo-labels inferred by clustering the domain discriminative features. They train a domain classifier against these pseudo-labels, which is further used to adversarially train a domain-invariant feature extractor. \cite{motiian2017unified} use a semantic alignment loss as an additional regularizer while training the classifier for domain invarient feature learning. \cite{huang2020self} iteratively locate dominant features activated on the training data using layer gradients, and learn useful features by self-challenging. \cite{wang2020learning} learn how to generalize across domains by simultaneously providing extrinsic supervision in the form of a metric learning task and intrinsic supervision in terms of a self-supervised auxiliary task. The most similar method to our own is \cite{sun2020test} in that they also run an inference-time procedure. However, unlike our method, they use the test sample for updating model parameters. 
\vspace{-0.31cm}
\section{Proposed Method}
\vspace{-0.1cm}
\begin{figure*}
\setlength\belowcaptionskip{-0.3cm}
\centering
  \scalebox{1}{
  \includegraphics[width=0.68\textwidth]{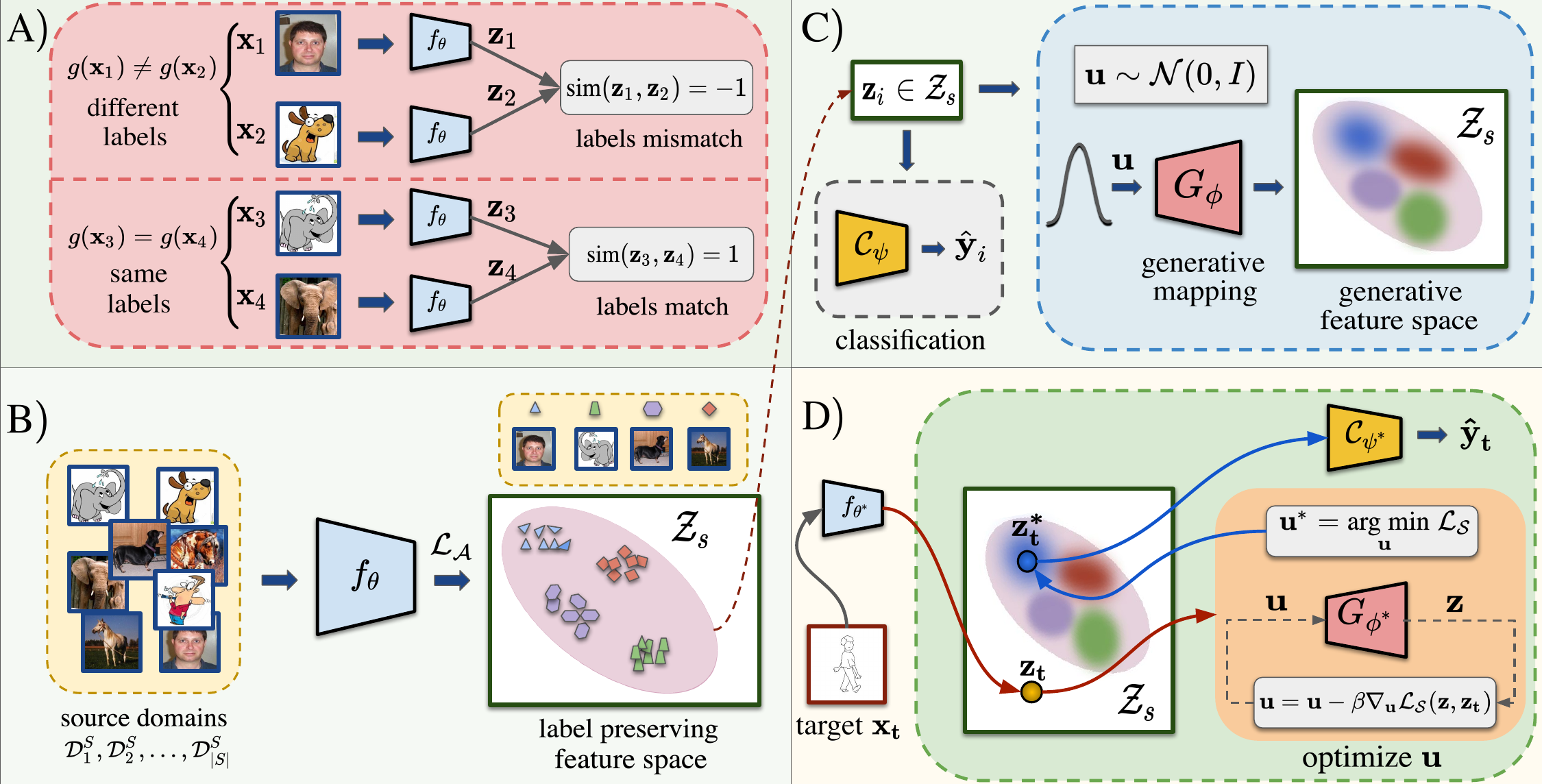}
  }
  \caption {\textbf{A)} We design a function $f$ (neural network $f_\theta$) to learn a label-preserving metric that produces a similarity score of 1 when the ground truth labels (given by function $g$) between a pair of images match and -1 otherwise. The function `sim' refers to the cosine similarity function. \textbf{B)} $f$ is implemented using a neural network $f_\theta$. During training, the examples from the source domains are utilized to create a source manifold $\mathcal{Z}_s$ using loss $\mathcal{L_A}$ such that the features on the manifold are implicitly \textit{clustered} to preserve the labels of examples. \textbf{C)} A classifier $\mathcal{C}_\psi$ and a generative model $G_\phi$ are trained on the label-preserving features from manifold $\mathcal{Z}_s$ such that $G_\phi$ learns to map a Gaussian vector $\mathbf{u}$ to a point on the manifold $\mathcal{Z}_s$. \textbf{D)} During inference, $f_{\theta^\ast}$ projects target $\mathbf{x_t}$ to a point  $\mathbf{z_t}$ on the label-preserving feature space. We propose an inference-time procedure to project the target feature to a point $\mathbf{z_t^\ast}$ on the source manifold which is finally classified to predict its label $\mathbf{\hat{y}_t}$. $\theta^\ast$, $\psi^\ast$ and $\phi^\ast$  indicate that the weights of their corresponding  networks are fixed during inference.}
  \label{fig:overall_train}
\end{figure*}

\subsection{Problem Setting and Method Overview}
Let $\mathcal{X}$ and $\mathcal{Y}$ respectively denote data and the label spaces. Let $\mathcal{H}$ be the space of hypotheses where each hypothesis $h$ in $\mathcal{H}$ maps points from $\mathcal{X}$ to a label in $\mathcal{Y}$. A domain is defined by the tuple $(\mathcal{D}, g_\mathcal{D})$ where $\mathcal{D}$ is a probability distribution over $\mathcal{X}$ and $g_\mathcal{D}$ where $g_\mathcal{D}:\mathcal{X} \to \mathcal{Y}$ is a function that assigns  ground-truth labels. It is generally assumed that the ground-truth labeling function $g$ is same across all the domains. Domain Generalization is defined as the task where there are a total of $N$ domains out of which $|S|$ are source and $|T|$ are target domains. The source and target domains are respectively denoted by $\mathcal{D}_{i}^{S}$,  $i \in [|S|]$ and $\mathcal{D}_{j}^{T}$, $j \in [|T|]$. The objective is to train a classifier on the source domains that predicts well on the target domain when the target samples are not available during training.

The motivation for our method comes from the following observation: DG methods that learn domain invariant representations do so only using the source data. Therefore, classifiers trained on such representations are not guaranteed to perform well on target data that is outside the source data manifolds. Hence, performance on the target data can be improved if the target sample is projected on to the manifold of the source features such that the ground-truth label is preserved, before classification. To this end, we propose a three-part procedure for domain generalization:\\
1. \textbf{(Training):} Learn a label-preserving domain invariant representation using source data. We first transform the data from multiple source domains into a space where they are clustered according to class labels, irrespective of the domains and build a classifier on these features.\\
2. \textbf{(Training):} Learn to generate features from the domain invariant feature manifold created from the source data by constructing a generative model on it.\\ 
3. \textbf{(Inference):} Given a test target sample, project it on to the source-feature manifold in a label-preserving manner. This is done by solving an inference-time optimization problem on the input space of the aforementioned generative model. Finally, classify the projected target feature.
    
Note that the parameters of all the networks involved (Feature extractor, generative model and the classifier) are learned only using the source data and fixed during inference. Thus, our approach is well within the realm of DG despite a label-preserving optimization problem being solved for every target sample during inference. The overall procedure is depicted in Figure \ref{fig:overall_train}. In the subsequent sections, we describe all the aforementioned components in detail.
\vspace{-0.2cm}
\subsection{Label-Preserving Transformation}
\vspace{-0.2cm}
\subsubsection{Domain invariant Features}
\vspace{-0.15cm}
The first step in our method is to learn a feature (metric) space such that the input images are clustered in accordance with their labels (classes) irrespective of their domains.   explicitly construct such a feature space $\mathcal{F}$ with a function $f:\mathcal{X} \to \mathcal{F}$ by solving the following optimization problem.
\vspace{-0.2cm}
{\small
\begin{argmini}
  {\tilde{f}}{\sum_{j=1}^{N}{{\sum_{i=1}^{N}{ (-1)^{\alpha(i,j)}\norm{\tilde{f}(\mathbf{x}_{i}) - \tilde{f}(\mathbf{x}_{j})}^{2}}}}}
    {}{}
    \addConstraint{\norm{\tilde{f}(\mathbf{x}_{k})}}{=1\ \forall k \in [N]}
\label{addm_opt}
\end{argmini}
\vspace{-0.3cm}
\begin{align*}
\alpha(i,j) = \begin{cases} 
          0 & g(\mathbf{x}_{i})=g(\mathbf{x}_{j}) \\
          1 & \text{otherwise} 
       \end{cases}
\end{align*}
}
 The function $f$ is learned such that when a pair of source samples have the same ground truth labels, the norm of the difference between their representations under $f$ is low irrespective of their domain membership and high when they belong to different classes. Under this formulation, the features in the space $\mathcal{F}$ will be `clustered' in accordance with their class labels, irrespective of the domains. In fact, it can be shown that the aforementioned  $f$ minimizes the $\mathcal{H}$-divergence \cite{ben2010theory} between any two pairs of domains.
\begin{prop}

The label-preserving transformation $f$ defined in Eq. 1, reduces the $\mathcal{H}$-divergence between any two pair of domains on which it is learned (Proof in Appendix).
\end{prop}
\vspace{-0.15cm}
In summary, the proposed feature transformation merges multiple source domains into a single feature domain such that images that have the same labels cluster into a group in the feature space. 
\vspace{-0.55cm}
\subsubsection{Learning the $f$-function}
\vspace{-0.23cm}
We propose to learn ${f}$ by parameterizing it with a deep neural network, $f_\theta$. It is easy to see that the objective function in Eq. \ref{addm_opt} reduces to an optimization of cosine-similarity between the pair of samples $f_\theta(\mathbf{x}_{i})$ and $f_\theta(\mathbf{x}_{j})$. Suppose $\mathbf{z}_i=f_{\theta}(\mathbf{x}_{i})$ and $\mathbf{z}_j=f_{\theta}(\mathbf{x}_{j})$ represent the feature vectors of inputs,  the cosine-similarity $\mathbf{s}_{i,j}$ between $\mathbf{z}_i$ and $\mathbf{z}_j$ is given by,
\begin{equation}
\label{addm_sim_loss}
    \mathbf{s}_{i,j} = \frac{\mathbf{z}_i \cdot \mathbf{z}_j}{\left\lVert \mathbf{z}_i \right\rVert\left\lVert \mathbf{z}_j \right\rVert} 
\end{equation}
Note that the optimization problem in Eq. \ref{addm_opt} seeks $\mathbf{s}_{i,j}$ to be high when the labels are same and low when they are different (denoted by the $\alpha(i,j)$ term in Eq. \ref{addm_opt}). Thus, we first translate $\mathbf{s}_{i,j}$ into logits for a sigmoid activation and use binary cross entropy on the generated probabilities.  However, since $-1 \leq \mathbf{s}_{i,j} \leq 1$, we scale it with a small positive constant $\tau$ (typically 0.1) to widen the range of the generated logits. Mathematically, we can write $\mathbf{p}_{i,j} = \text{sigmoid}(\mathbf{s}_{i,j}/\tau)$. Under this formulation, one can treat $\mathbf{p}_{i,j}$ as a similarity score which should be 1 if ($ \mathbf{x}_i, \mathbf{x}_j$) have the same label ($\alpha(i,j)=0$) and 0 otherwise. Thus, we finally use a binary cross entropy loss $\mathcal{L_A}$ between $\mathbf{p}_{i,j}$ and $1 -\alpha(i,j)$, to train the $f_\theta$ network. 
\vspace{-0.1cm}
\subsection{Inference-time Target Projections}
\vspace{-0.1cm}
In a DG setting, the feature transformation mentioned in the previous sections is learned on the source domains. Let $\mathcal{Z}_s$ denote the manifold created by learning such features using the source data. The classifier $\mathcal{C}_\psi$ is trained on points from the source data feature manifold $\mathcal{Z}_s$. Because of the domain shift, the feature $f(\mathbf{x_t})$ corresponding to a test target point $\mathbf{x_t}$ might not fall on the source data feature manifold $\mathcal{Z}_s$. This causes the classifier to fail on the target feature $f(\mathbf{x_t})$. To address this issue, we propose to project or `push' the target feature onto $\mathcal{Z}_s$ while preserving the ground truth label of $\mathbf{x_t}$, so that the classifier can better discern the class label of $\mathbf{x_t}$. We propose to accomplish such a label-preserving projection by solving an inference-time optimization problem on the input space of a generative model trained on $\mathcal{Z}_s$. For convenience, we define a function $\tilde{g}: \mathcal{F} \to \mathcal{Y}$ such that $\tilde{g}(f(x)) = g(x)$

\vspace{-0.35cm}
\subsubsection{Generating the Source-Feature Manifold $\mathcal{Z}_s$}
\vspace{-0.2cm}
Once the label-preserving source-feature manifold $\mathcal{Z}_s$ is obtained through $f_\theta$, we build a generative model on it. That is, a transformation from samples from an arbitrary distribution, such as Normal distribution, to the source-data feature manifold $\mathcal{Z}_s$ is learned. We choose two state-of-the-art neural generative models for this purpose: (a) Variational Auto-Encoder (VAE) \cite{kingma2013auto}: In this setting, a VAE is trained using the source-data features $\mathbf{z}$, by encoding them to produce the latent space $\mathbf{u} \sim \mathcal{N}(0,I)$. A decoder $G_\phi$ reconstructs (generates) the source-feature manifold $\mathcal{Z}_s$ by minimizing a regularized norm-based loss. (b) Generative Adversarial Networks (GAN) \cite{goodfellow2014generative}: Here, a GAN is trained with a generator network $G_\phi$ that maps an arbitrary latent space $\mathbf{u} \sim \mathcal{N}(0,I)$ to the source-feature manifold $\mathcal{Z}_s$. Note that these generative models are trained on the source-features alone and fixed during the inference procedure. We denote the trained generative model by $G_{\phi^\ast}$.
\vspace{-0.4cm}
\subsubsection{Label-preserving Projections}
\vspace{-0.1cm}
The final component of our method is to project the target features on to the source-feature manifold during inference. It is to be noted that the transformation $f_\theta$ is constructed such that when a pair of samples have zero distance in that space, they have same ground truth label. That is, if $\lVert f_\theta(\mathbf{x}_1)-f_\theta(\mathbf{x}_2) \rVert = 0$, then $g(\mathbf{x}_1)=g(\mathbf{x}_2)$. We exploit this property and solve a (per-sample) optimization procedure on the input space of the generative model $G_{\phi^\ast}(\mathbf{u})$ to obtain the target-feature projection. 

Let $\mathbf{z_t}=f_\theta(\mathbf{x_t})$ denote the feature vector corresponding to a test target sample $\mathbf{x_t}$. Our goal is to find the projected target feature in the source-feature manifold $\mathbf{z_t^*} \in \mathcal{Z}_s$ that has the same ground-truth label as that of $\mathbf{z_t}$. By construction of $f_\theta$ (Eq. \ref{addm_opt}), the cosine distance between $\mathbf{z_t}$ and $\mathbf{z_t^*}$ should be low if their ground truth labels are to match. Based on this, we devise the following optimization problem on the input space of $G_{\phi^*}$ to find $\mathbf{z_t^*}$:
\vspace{-0.2cm}
{\small
\begin{align}
\mathcal{L_S} &= \Bigg[1 - {\frac{\mathbf{z_t} \cdot G_{\phi^\ast}(\mathbf{u})} { \left\lVert \mathbf{z_t} \right\rVert \left\lVert G_{\phi^\ast}(\mathbf{u}) \right\rVert} }\Bigg]\\
&\mathbf{u^*} = \argmin_{{\mathbf{u}}} \mathcal{L_S} \label{eqn:latent_opt}\\
&\mathbf{z_t^*} = G_{\phi^\ast}(\mathbf{u^*})
\label{eqn:proj_point}
\end{align}
}
The objective function in Eq. \ref{eqn:latent_opt} seeks to find the projected target feature $\mathbf{z_t^*}$ (via $\mathbf{u^*}$ from Eq. \ref{eqn:proj_point})  whose cosine distance is least from the true target feature $\mathbf{z_t}$. The implicit assumption of our method is that when the target example is projected on to the source manifold, minimizing distance is equivalent to preserving labels.
\vspace{-0.3cm}
\subsubsection{Analysis and Implementation} \label{section:analysis_and_implementation}
\vspace{-0.2cm}
In this section, we analyze the performance of the classifier on account of it using the projected target instead of the real target features. We start by upper bounding the expected value of the misclassification when the projected target is used in the classifier, in the below preposition. 

\begin{prop}
The expected misclassification rate obtained with a classifier $h$ when the projected target is used instead of the true target, obeys the following upper-bound:
\begin{align}
&\mathbb{E}_{(\mathcal{D}^T, \mathcal{D}^{T^\ast})} \left| \tilde{g}(\mathbf{z_t}) - h(\mathbf{z_t^*}) \right| \leq \nonumber \\
    & \underbrace{\mathbb{E}_{\mathcal{D}^{T^\ast}} \left| \tilde{g}(\mathbf{z_t^*}) - h(\mathbf{z_t^*}) \right|}_\text{i} + \underbrace{\mathbb{E}_{(\mathcal{D}^T, \mathcal{D}^{T^\ast})} \left| \tilde{g}(\mathbf{z_t}) - \tilde{g}(\mathbf{z_t^*}) \right|}_\text{ii}
    \label{eqn:risk_ineq}
\end{align}
where $\mathcal{D}$ and $\mathcal{D}^{T^\ast}$ respectively denote the true and the projected target distributions respectively (Proof in Appendix).
\end{prop}

The term (i) in Eq. \ref{eqn:risk_ineq} is the misclassification error of $h$ on the projected target and term (ii) is the difference between the ground truth labels of the true and the projected targets. Given that our overall objective is to minimize the LHS of Eq. \ref{eqn:risk_ineq}, the optimization procedure in the previous section aims to minimize term (ii) while term (i) is expected to be less since the projected target is expected to lie on the source-feature manifold $\mathcal{Z}_s$ on which the classifier is trained.

In the implementation, during inference, we optimize the objective in Eq. \ref{eqn:latent_opt} (reducing term (ii) in Eq. \ref{eqn:risk_ineq}) by gradient descent. A discussion on the choice of stopping criteria can be found in Section \ref{section:stopping_criteria}. The training and the inference procedures are detailed in Algorithm 1 and shown in Figure \ref{fig:overall_train}. 

\begin{algorithm}
\footnotesize
\SetAlgoLined
\vspace{2pt}
\textbf{Training}\\
\vspace{-5pt}
\hrulefill\\
\textbf{Input:} Batch size $N$, learning rate $\eta$, source data $\{(\mathbf{x}_k, \mathbf{y}_k)\}$;\\
\KwResult{Trained $f_{\theta^\ast}$, generative model $G_{\phi^\ast}$, classifier $\mathcal{C_{\psi^\ast}}$ }
 \For{ \text{sampled minibatch} $\{(\mathbf{x}_k, \mathbf{y}_k)\}^N_{k=1}$ } {
  
\For{all $i \in \{1,...N\}$ and $j \in \{1,...N\}$ } {
$(\mathbf{z}_i, \mathbf{z}_j)  \leftarrow (f_{\theta}(\mathbf{x}_i), f_{\theta}(\mathbf{x}_j))$ \\ 
\vspace{2pt}
$ \mathbf{s}_{i,j} \leftarrow \frac{\mathbf{z}_i \cdot \mathbf{z}_j} { \lVert \mathbf{z}_i \rVert \lVert \mathbf{z}_j \rVert} $ \\
\vspace{2pt}
$\mathbf{y}_{i,j} = \delta_{\mathbf{y}_i, \mathbf{y}_j}$

$\mathbf{p}_{i,j}\leftarrow \text{sigmoid}(\mathbf{s}_{i,j}/\tau)$\\
}

$\mathcal{L_\mathcal{A}} \leftarrow \frac{1}{N^2} \sum_{i=1}^{N} \sum_{j=1}^{N} \text{BCELoss}(\mathbf{p}_{i,j}, \mathbf{y}_{i,j})$\\
\vspace{2pt}
$\theta \leftarrow \theta - \eta \nabla_{\theta} \mathcal{L_A}$\\
}
Train $G_\phi$ and $\mathcal{C_\psi}$ on $\{(f_\theta(\mathbf{x}_k), \mathbf{y}_k)\}$.\\
\hrulefill\\
\textbf{Inference}\\
\vspace{-5pt}
\hrulefill\\
\textbf{Input:} Target image $\mathbf{x_t}$, trained network $f_{\theta^\ast}$, generative model $G_{\phi^\ast}$, classifier $\mathcal{C_{\psi^\ast}}$, iteration rate $\beta$;\\
\KwResult{Target label $\mathbf{\hat{y}_t}$}
$\mathbf{z_t} \leftarrow f_{\theta^\ast}(\mathbf{x_t})$; \\
 Sample $\mathbf{u}$ from $\mathcal{N}(0,I)$;\\
 Initialize $U$ and $L$ as empty lists\\
\For{all $i \in \{1,...M\}$} {
  $\mathbf{z} \leftarrow G_{\phi^\ast} (\mathbf{u})$\\
 $\mathcal{L_S} \leftarrow 1 - \frac{\mathbf{z} \cdot \mathbf{z_t}} { \left\lVert \mathbf{z} \right\rVert \left\lVert \mathbf{z_t} \right\rVert}$\\
 \vspace{2pt}
 $\mathbf{u} \leftarrow \mathbf{u} - \beta \nabla_{\mathbf{u}} \mathcal{L_S}$\\
 $(U[i], L[i]) \leftarrow (\mathbf{u}, \mathcal{L_S})$\\
 }
 Smoothen $L$ by window-averaging\\
 $\mathbf{u}^\ast \leftarrow U[\argmax_{{i}} \delta^2 L]$\\
 $\mathbf{\hat{y}_t} \leftarrow \mathcal{C_{\psi^\ast}}(G_{\phi^\ast}(\mathbf{u}^\ast))$
\caption{Inference-time Target Projections}

\end{algorithm}

\setlength{\textfloatsep}{1\baselineskip plus 0.2\baselineskip minus 0.5\baselineskip}
\vspace{-0.4cm}
\section{Experiments and Results}
\vspace{-0.12cm}
We have considered four standard DG datasets - PACS \cite{li2017deeper}, VLCS \cite{Fang_2013_ICCV}, Office-Home \cite{venkateswara2017deep} and Digits-DG \cite{zhou2020deep} to demonstrate the efficacy of our  method. All these datasets contain four domains out of which three are used as sources and the other as a target in a leave-one-out strategy.  We use a VAE as the generative model $G_\phi$ in all our main results owing to its stability of training vis-a-vis a GAN. The performance metric is the classification accuracy and we compare against a baseline Deep All method: classifier trained on the combined source domains without employing any DG techniques. For each target domain, we have reported our average and standard deviation for five independent runs of the model. We also compare our method with the existing DG methods and report the results, dataset wise. We report the standard deviation as 0 for models which have not reported them. For each dataset, we use the validation set for selecting hyperparameters if it is available. Otherwise, we split the data from source domains and use the smaller set for hyperparameter selection. We also use data augmentation for regularizing the network $f_\theta$. For further details on the datasets, machine configuration and choice of hyperparameters, please refer to the Appendix. 

\setlength{\textfloatsep}{1\baselineskip plus 0.2\baselineskip minus 0.5\baselineskip}
\begin{table}[hbt!]
\centering
  \scalebox{0.69}{
  \begin{tabular}{lccccc}
    \toprule
        {Method} & {Art.}
        & {Cartoon}
        & {Sketch} & {Photo} & {Avg.}
        \\
    \midrule
    \multicolumn{6}{c}{AlexNet}\\
    \midrule
    Deep All&65.96$\pm$0.2&69.50$\pm$0.2&59.89$\pm$0.3&89.45$\pm$0.3&71.20\\
    Jigen \cite{carlucci2019domain}&67.63$\pm$0.0&71.71$\pm$0.0&65.18$\pm$0.0&89.00$\pm$0.0&73.38\\

    MMLD \cite{matsuura2020domain}&69.27$\pm$0.0&72.83$\pm$0.0&66.44$\pm$0.0&88.98$\pm$0.0&74.38\\
    MASF \cite{dou2019domain}&70.35$\pm$0.3&72.46$\pm$0.2&67.33$\pm$0.1&90.68$\pm$0.1&75.21\\
    EISNet \cite{wang2020learning}&70.38$\pm$0.4&71.59$\pm$1.3&70.25$\pm$1.4&91.20$\pm$0.0&75.86\\
    RSC \cite{huang2020self}&71.62$\pm$0.0&75.11$\pm$0.0&66.62$\pm$0.0&90.88$\pm$0.0&76.05\\
     Ours&\textbf{72.67$\pm$0.5}&\textbf{76.51$\pm$0.3}&\textbf{73.09$\pm$0.2}&\textbf{92.01$\pm$0.3}&\textbf{78.57}\\
    \midrule
    \multicolumn{6}{c}{ResNet-18}\\
    \midrule
    Deep All&77.65$\pm$0.2&75.36$\pm$0.3&69.08$\pm$0.2&95.12$\pm$0.1&79.30\\
    MMLD \cite{matsuura2020domain}&81.28$\pm$0.0&77.16$\pm$0.0&72.29$\pm$0.0&96.09$\pm$0.0&81.83\\
    EISNet \cite{wang2020learning}&81.89$\pm$0.9&76.44$\pm$0.3&74.33$\pm$1.4&95.93$\pm$0.1&82.15\\
    L2A-OT \cite{zhou2020learning} &83.30$\pm$0.0&78.20$\pm$0.0&73.60$\pm$0.0&96.20$\pm$0.0&82.80\\
    DSON \cite{seo2019learning} &84.67$\pm$0.0&77.65$\pm$0.0&\textbf{82.23$\pm$0.0}&95.87$\pm$0.0&85.11\\
    RSC \cite{huang2020self}&83.43$\pm$0.8&80.31$\pm$1.8&80.85$\pm$1.2&95.99$\pm$0.3&85.15\\
    Ours&\textbf{86.39$\pm$0.3}&\textbf{81.26$\pm$0.2}&81.79$\pm$0.1&\textbf{97.15$\pm$0.4}&\textbf{86.65}\\
    \midrule
    \multicolumn{6}{c}{ResNet-50}\\
    \midrule
    Deep All&81.31$\pm$0.3&78.54$\pm$0.4&69.76$\pm$0.4&94.97$\pm$0.1&81.15\\
    MASF \cite{dou2019domain}&82.89$\pm$0.2&80.49$\pm$0.2&72.29$\pm$0.2&95.01$\pm$0.1&82.67\\
    EISNet \cite{wang2020learning}&86.64$\pm$1.4&81.53$\pm$0.6&78.07$\pm$1.4&97.11$\pm$0.4&85.84\\
    DSON \cite{seo2019learning} &87.04$\pm$0.0&80.62$\pm$0.0&82.90$\pm$0.0&95.99$\pm$0.0&86.64\\
    RSC \cite{huang2020self}&87.89$\pm$0.0&82.16$\pm$0.0&83.85$\pm$0.0&97.92$\pm$0.0&87.83\\
    Ours&\textbf{90.25$\pm$0.4}&\textbf{85.19$\pm$0.2}&\textbf{86.20$\pm$0.5}&\textbf{98.97$\pm$0.1}&\textbf{90.15}\\
    
    \bottomrule
  \end{tabular}
  }
  \caption{Comparison of performance between different models on PACS \cite{li2017deeper} dataset with AlexNet, ResNet-18 and ResNet-50 as backbones for the $f_\theta$ network.}
  \label{table:pacs}
\end{table}
\vspace{-0.35cm}
\subsection{Multi-source Domain Generalization}
\vspace{-0.1cm}
\textbf{PACS:} The PACS dataset consists of images from \textbf{P}hoto, \textbf{A}rt Painting, \textbf{C}artoon  and \textbf{S}ketch  domains. We follow the experimental protocol defined in \cite{li2017deeper}. We use ResNet-50, ResNet-18 and AlexNet as backbones for the feature extractor network $f_\theta$ and train them on source domains. To perform target projection, we learn to sample from the features produced by $f_\theta$ by training a VAE on the latent space ($\mathbf{u}$). We achieve state-of-the-art results with all three choices of backbones as shown in Table \ref{table:pacs}. 
\setlength{\textfloatsep}{1\baselineskip plus 0.2\baselineskip minus 0.5\baselineskip}
\begin{table}[hbt!]
\centering
  \scalebox{0.69}{
  \begin{tabular}{lccccc}
    \toprule
        {Method} & {Caltech}
        & {LabelMe}
        & {Pascal} & {Sun} & {Avg.}
        \\
    \midrule
    Deep All&96.45$\pm$0.1&60.03$\pm$0.5&70.41$\pm$0.4&62.63$\pm$0.3&72.38\\
    Jigen \cite{carlucci2019domain}&96.93$\pm$0.0&60.90$\pm$0.0&70.62$\pm$0.0&64.30$\pm$0.0&73.19\\

    MMLD \cite{matsuura2020domain}&96.66$\pm$0.0&58.77$\pm$0.0&71.96$\pm$0.0&68.13$\pm$0.0&73.88\\
    
    MASF \cite{dou2019domain}&94.78$\pm$0.2&64.90$\pm$0.1&69.14$\pm$0.2&67.64$\pm$0.1&74.11\\
    EISNet \cite{wang2020learning}&97.33$\pm$0.4&63.49$\pm$0.8&69.83$\pm$0.5&68.02$\pm$0.8&74.67\\
    RSC \cite{huang2020self}&97.61$\pm$0.0&61.86$\pm$0.0&73.93$\pm$0.0&68.32$\pm$0.0&75.43\\
    Ours&\textbf{98.12$\pm$0.1}&\textbf{66.80$\pm$0.3}&\textbf{74.77$\pm$0.4}&\textbf{70.43$\pm$0.1}&\textbf{77.53}\\
    
    
    \bottomrule
  \end{tabular}
  }
  \caption{Comparison of performance between different models using AlexNet backbone on VLCS \cite{Fang_2013_ICCV} dataset.}
  \label{table:vlcs}
\end{table}

\textbf{VLCS:} VLCS comprises of the \textbf{V}OC2007(Pascal), \textbf{L}abelMe, \textbf{C}altech and \textbf{S}un domains, all of which contain photos. We follow the same experimental setup as mentioned in \cite{matsuura2020domain} where we train on three source domains with 70\% data from each and test on all the examples from the fourth target domain. We utilize similar setup for $f_\theta$ and $G_\phi$ as in the PACS dataset with AlexNet as the backbone.  We achieve SOTA results on VLCS as evident from Table \ref{table:vlcs}. We emphasize that unlike PACS dataset where the domains differ in image styles, VLCS consists of domains that contain only photos. Thus, we demonstrate that our method generalizes well even when the source domains are not diverse. 
\vspace{-0.15cm}
\begin{table}[hbt!]
\centering
  \scalebox{0.69}{
  \begin{tabular}{lccccc}
    \toprule
        {Method} & {Artistic}
        & {Clipart}
        & {Product} & {Real-World} & {Avg.}
        \\
    \midrule
    Deep All&52.06$\pm$0.5&46.12$\pm$0.3&70.45$\pm$0.2&72.45$\pm$0.2&60.27\\
    
    D-SAM \cite{d2018domain}&58.03$\pm$0.0&44.37$\pm$0.0&69.22$\pm$0.0&71.45$\pm$0.0&60.77\\
    Jigen \cite{carlucci2019domain}&53.04$\pm$0.0&47.51$\pm$0.0&71.47$\pm$0.0&72.79$\pm$0.0&61.20\\
   MMD-AAE \cite{li2018domain}&56.50$\pm$0.0&47.30$\pm$0.0&72.10$\pm$0.0&74.80$\pm$0.0&62.70\\
    DSON \cite{seo2019learning} &59.37$\pm$0.0&45.70$\pm$0.0&71.84$\pm$0.0&74.68$\pm$0.0&62.90\\
    RSC \cite{huang2020self}&58.42$\pm$0.0&47.90$\pm$0.0&71.63$\pm$0.0&74.54$\pm$0.0&63.12\\
    
    L2A-OT \cite{zhou2020learning} &60.60$\pm$0.0&50.10$\pm$0.0&74.80$\pm$0.0&77.00$\pm$0.0&65.60\\
    Ours&\textbf{62.63$\pm$0.2}&\textbf{55.79$\pm$0.3}&\textbf{76.86$\pm$0.1}&\textbf{78.98$\pm$0.1}&\textbf{68.56}\\
    
    
    \bottomrule
  \end{tabular}
  }
  \caption{Comparison of performance between different models using ResNet-18 backbone on Office-Home \cite{venkateswara2017deep}.}
  \label{table:officehome}
\end{table}
\setlength{\textfloatsep}{1\baselineskip plus 0.2\baselineskip minus 0.5\baselineskip}
\vspace{-0.5cm}
\begin{table}[hbt!]
\centering
  \scalebox{0.69}{
  \begin{tabular}{lccccc}
    \toprule
        {Method} & {MNIST}
        & {MNIST-M}
        & {SVHN} & {SYN} & {Avg.}
        \\
    \midrule
    Deep All&95.24$\pm$0.1&58.36$\pm$0.6&62.12$\pm$0.5&78.94$\pm$0.3&73.66\\
    Jigen \cite{carlucci2019domain}&96.50$\pm$0.0&61.40$\pm$0.0&63.70$\pm$0.0&74.00$\pm$0.0&73.90\\
    CCSA \cite{motiian2017unified}&95.20$\pm$0.0&58.20$\pm$0.0&65.50$\pm$0.0&79.10$\pm$0.0&74.50\\
    
   MMD-AAE \cite{li2018domain}&96.50$\pm$0.0&58.40$\pm$0.0&65.00$\pm$0.0&78.40$\pm$0.0&74.60\\
    
    
    CrossGrad \cite{shankar2018generalizing}&96.70$\pm$0.0&61.10$\pm$0.0&65.30$\pm$0.0&80.20$\pm$0.0&75.80\\
    L2A-OT \cite{zhou2020learning} &96.70$\pm$0.0&63.90$\pm$0.0&68.60$\pm$0.0&83.20$\pm$0.0&78.10\\
    Ours&\textbf{97.99$\pm$0.1}&\textbf{66.52$\pm$0.4}&\textbf{71.31$\pm$0.3}&\textbf{85.40$\pm$0.5}&\textbf{80.30}\\
    
    
    \bottomrule
  \end{tabular}
  }
  \caption{Comparison of performance between different models on Digits-DG \cite{zhou2020deep} dataset.}
  \label{table:digitsDG}
\end{table}

\vspace{-0.5cm}
\textbf{Office-Home:} Office-Home contains images from 4 domains namely Artistic, Clipart, Product and Real-World. We follow the experimental protocol as outlined in \cite{d2018domain}.  We utilize a ResNet-18 as backbone with two additional fully-connected layers. 
The reconstruction error in VAE is minimized using both L1 and L2 loss functions. We achieve SOTA results on all the domains using a ResNet-18 backbone as shown in Table \ref{table:officehome}. It should be noted that Clipart is a difficult domain to generalize to as it is dissimilar to other domains. We achieve 5.6\% improvement over the nearest competitor on the Clipart domain, showcasing the merit of our method in generalizing to dissimilar target domains.

\textbf{Digits-DG:} Digits-DG is the task of digit recognition using MNIST \cite{lecun1998gradient}, MNIST-M \cite{ganin2015unsupervised}, SVHN \cite{netzer2011reading} and SYN \cite{ganin2015unsupervised} domains that differ drastically in font style and background. We follow the experimental setup of \cite{zhou2020deep} and use their architecture for the feature extractor $f_\theta$. The $G_\phi$ network is implemented in a similar way as it is implemented for the PACS dataset with ResNet-18 as the backbone. Table \ref{table:digitsDG} shows the performance of our method on Digits-DG dataset compared against various SOTA methods.
\setlength{\textfloatsep}{1\baselineskip plus 0.1\baselineskip minus 0.1\baselineskip}

\begin{figure*}[tb!]
\captionsetup{skip=-10pt}
\setlength\belowcaptionskip{-0.3cm}
\centering
\subfloat[][]{\includegraphics[width=0.210\textwidth]{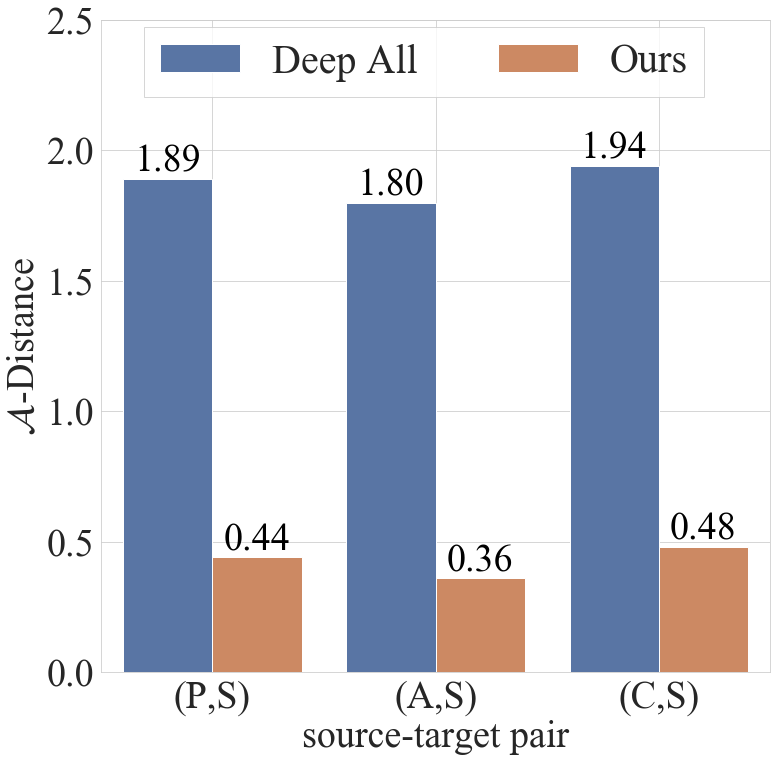}\label{fig:h_div}}
\hspace{0.34cm}
\subfloat[][]{\includegraphics[width=0.210\textwidth]{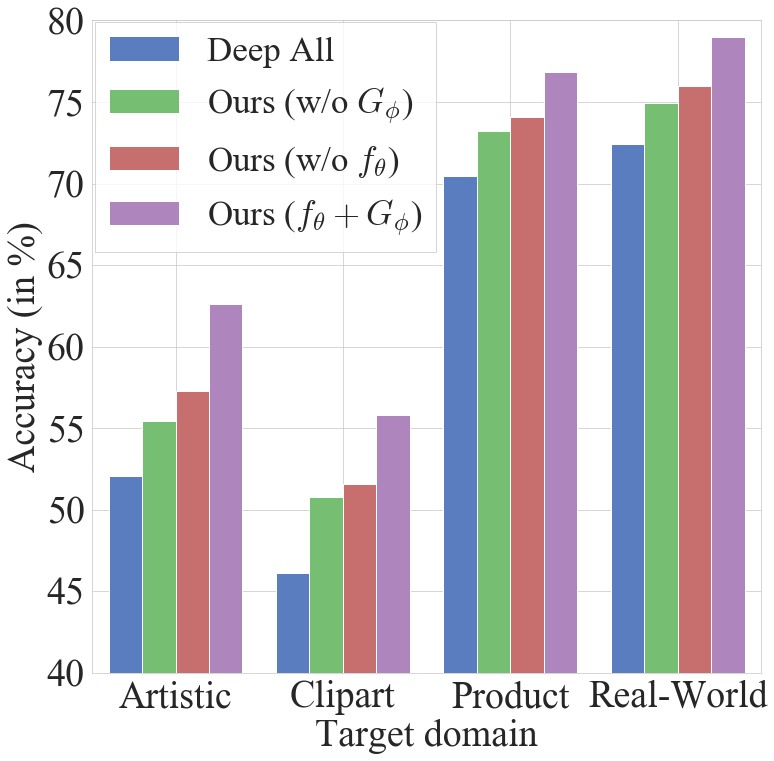}\label{fig:ablation}}
\hspace{0.34cm}
\subfloat[][]{\includegraphics[width=0.210\textwidth]{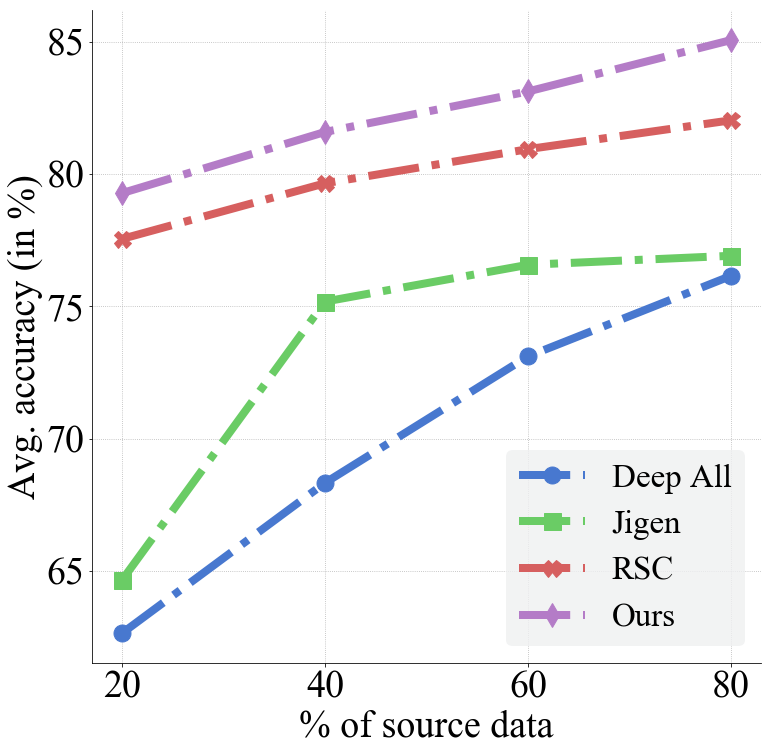}\label{fig:scarce_data}}
\hspace{0.34cm}
\subfloat[][]{\includegraphics[width=0.210\textwidth]{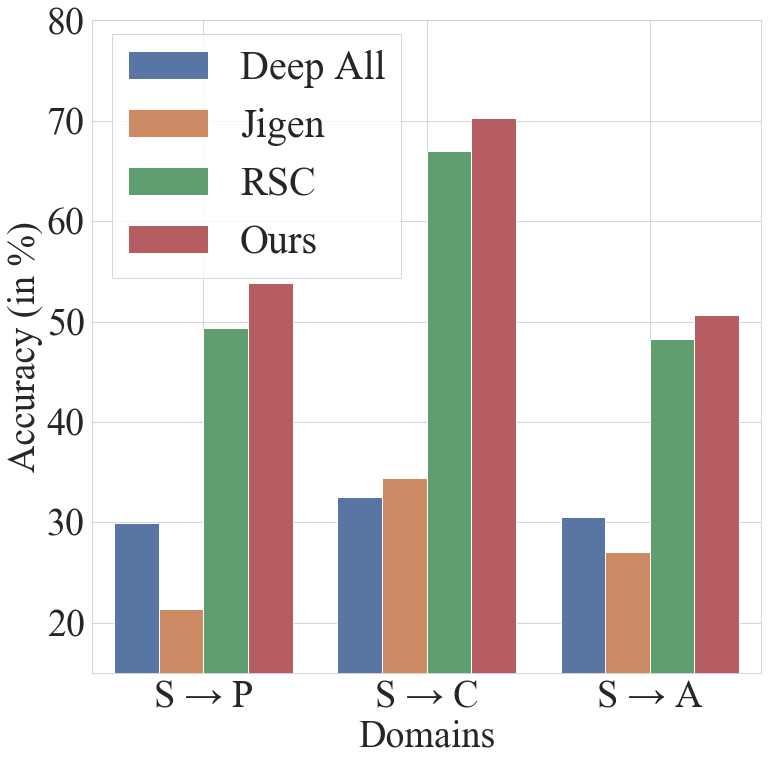}\label{fig:single_source_dg}}
\caption{\textbf{(left to right)} \textbf{(a)} Plot showing $\mathcal{A}$-divergence between source and projected target features. \textbf{(b)} Ablation on Office-Home dataset highlighting the relative importance of different components employed in our method. \textbf{(c)} Relative performance of different methods trained on 20\%, 40\%, 60\% and 80\% of PACS dataset. For each setting, we report average performance with all examples from the target domains using leave-one-out strategy. \textbf{(d)} Relative performance of DG methods trained on a single domain (Sketch (S)) and tested on Photo (P), Cartoon (C) and Art (A) domains.}
\label{fig:single_source_curves}
\end{figure*}
\vspace{-0.27cm}
\subsection{Pairwise $\mathcal{H}$-divergence}
\vspace{-0.23cm}
We examine the effectiveness of the proposed method in projecting target features onto the source-feature manifold by computing a proxy measure for $\mathcal{H}$-divergence called the $\mathcal{A}$-distance \cite{ben2007analysis}, that measures the divergence between two domains. We compute the $\mathcal{A}$-distance between features from the Sketch domain (target domain) and each of the three source domains (Photo, Cartoon and Art), obtained in two ways - first, from a ResNet-18 model trained on the source domains (Deep All), and second, the features obtained after performing target projections using our method. We compare the $\mathcal{A}$-distance obtained from these two methods in Figure \ref{fig:h_div} . It is observed that the source-target divergence is reduced substantially compared to Deep All, suggesting the effectiveness of our projection scheme in bringing the target points onto the source manifold.  
\setlength{\textfloatsep}{1\baselineskip plus 0.1\baselineskip minus 0.1\baselineskip}
\vspace{-0.25cm}
\subsection{Ablation Studies}
\vspace{-0.24cm}
To quantify the importance of the two components (i.e. the networks $f_\theta$ and $G_\phi$), we perform ablative studies using the Office-Home and Digits-DG datasets. For the Office-Home dataset, the Deep All model is trained with ResNet-18 as the backbone while the backbone proposed by \cite{zhou2020deep} is used for the Deep All model on the Digits-DG dataset. We classify target domain features extracted from the feature extractor network without the target projection procedure as shown in Figure \ref{fig:ablation}. We observe a marginal improvement over Deep All by classifying on the domain-invariant features learnt by $f_\theta$ without using the $G_\phi$ network. We attribute this to the explicit class separability imposed by the $f_\theta$ network in the feature space $\mathcal{F}$. When we exclude the $f_\theta$ network by training the generative model on Deep All features and perform the projection procedure on these features, we still observe a sizable performance boost over merely classifying with the Deep All model. This shows the effectiveness of the target projection procedure as a sampling strategy. The best performance is achieved with all components working together as seen in Figure \ref{fig:ablation}.
\setlength{\textfloatsep}{1\baselineskip plus 0.2\baselineskip minus 0.5\baselineskip}
\vspace{-0.3cm}
\subsection{Sampling strategy}
\vspace{-0.25cm}
To quantify the effect of the specific generative model used for the $G_\phi$ network, we report performance with three samplers on VLCS dataset, namely: (a) VAE, (b) 1-Nearest Neighbor (1-NN) of the target sample with the source samples using the similarity metric as in Eq. \ref{addm_sim_loss}, (c) GAN, and show results in Table \ref{table:vlcs_gans}. It is seen that sampling using continuous generative models (VAE and GAN) offer better performance compared to 1-NN sampling.  On domains like LabelMe and Sun, they offer improvement of about 5.7\% and 7.0\% respectively over the 1-NN method of sampling. This is attributed to the fact that generative models can approximate the true data distribution by learning from empirical data and offer infinite sampling while 1-NN search is restricted to the existing training points only.
\setlength{\textfloatsep}{1\baselineskip plus 0.1\baselineskip minus 0.1\baselineskip}
\begin{table}[hbt!]
\centering
  \scalebox{0.69}{
  \begin{tabular}{lccccc}
    \toprule
        {Method} & {Caltech}
        & {LabelMe}
        & {Pascal} & {Sun} & {Avg.}
        \\
    \midrule
    Deep All&96.45&60.03&70.41&62.63&72.38\\
    1-NN&96.51&61.44&71.82&63.46&73.31\\
    Ours ($G_\phi$ = GAN)&97.89&\textbf{67.18}&74.59&70.28&77.48\\
    Ours ($G_\phi$ = VAE)&\textbf{98.12}&66.80&\textbf{74.77}&\textbf{70.43}&\textbf{77.53}\\
    \bottomrule
  \end{tabular}
  }
  \caption{Performance of our method on VLCS dataset with different sampling strategies.}
  \label{table:vlcs_gans}
\end{table}
\vspace{-0.7cm}
\subsection{Low resource settings}
\vspace{-0.15cm}
  

  
In this section, we demonstrate the efficacy of our method in low-resource settings. We show that our method generalizes well when used in scarce resource settings and single domain DG problems, and can also easily be extended to supervised domain adaptation with access to scant labelled target samples. Our method has a two-fold advantage that makes it especially data-efficient: (a) since $f_\theta$ is trained on pairs of images, it can learn effectively even on small datasets (b) the generative model $G_\phi$ enables infinite sampling during the target projection procedure. 

\begin{figure}[tb!]
\centering
\captionsetup{skip=-7pt}
\subfloat[][]{\includegraphics[width=0.22\textwidth]{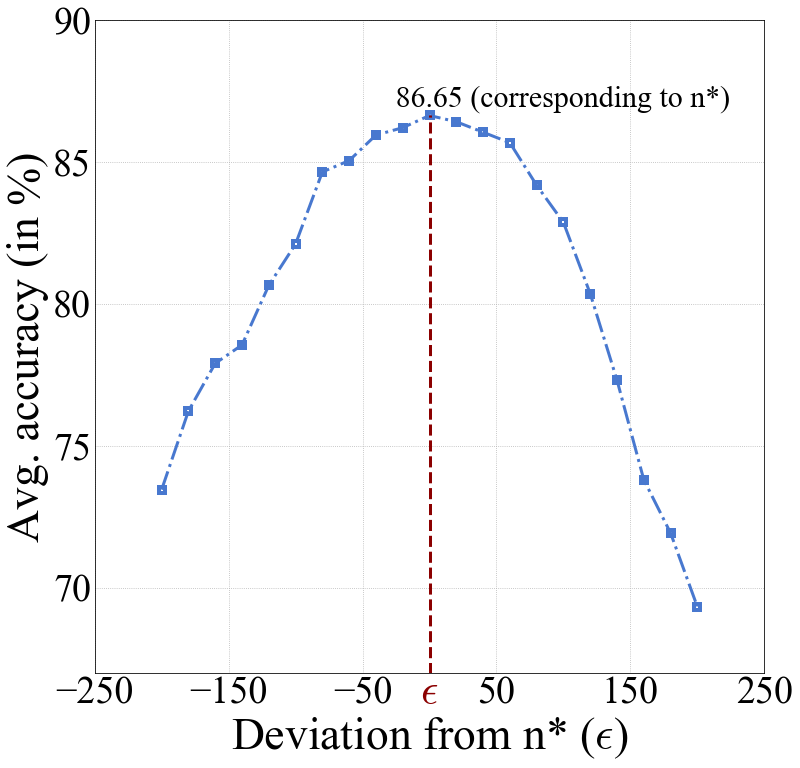}\label{fig:dev_nstar}\label{fig:loss_stop_point}}
\hspace{0.3cm}
\subfloat[][]{\includegraphics[width=0.22\textwidth]{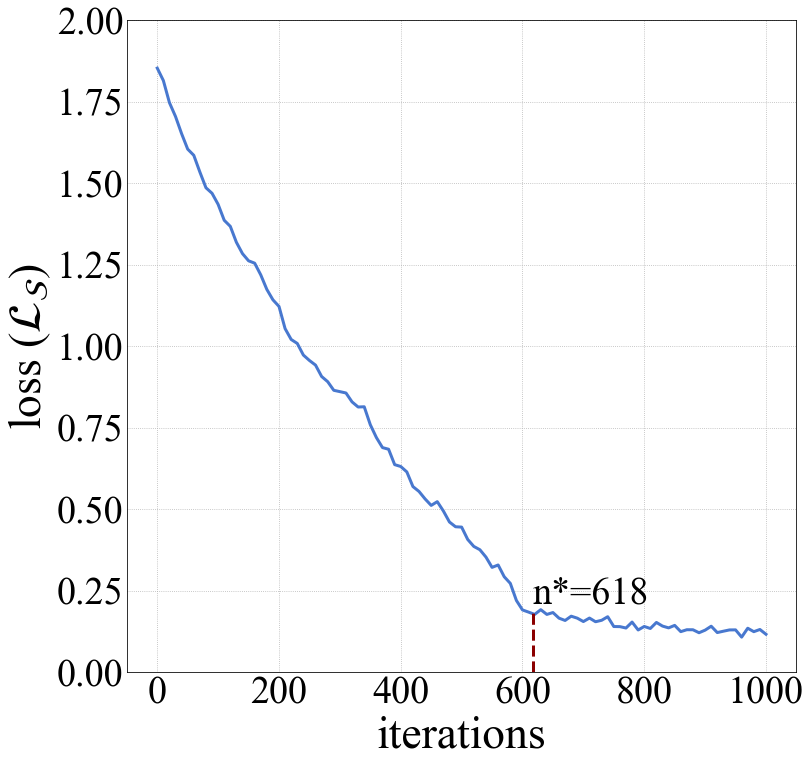}\label{fig:loss_elbow_curve}}
\vspace{-0.1cm}
\caption{\textbf{(left)} \textbf{(a)} Effect of deviation from the optimal number of iterations of each target example during inference on PACS. \textbf{(right)} \textbf{(b)} Plot of iterations vs loss for projection of a target example from Sketch domain on source manifold ($\mathcal{Z}_s$) created by Photo, Art and Cartoon domains.}

\end{figure}
\vspace{-0.48cm}
\subsubsection{Scarce resource setting}
\vspace{-0.32cm}
  We compare our method against Deep All and two state-of-the-art methods, RSC \cite{huang2020self} and Jigen \cite{carlucci2019domain}. For each source domain in PACS we train on \{20\%, 40\%, 60\%, 80\%\} data from each domain and test on the entire target domain. From Figure \ref{fig:scarce_data}, it is evident that our method outperforms all the baselines at every considered resource setting. 
  \vspace{-0.3cm}
 \subsubsection{Supervised Domain Adaptation}
 \vspace{-0.24cm}
  This setting has also been referred to as Few-shot Domain Adaptation in \cite{qiao2020learning} and \cite{motiian2017few}. In this setting, in addition to the source data, we assume that we have access to a limited number of labelled samples ($|\mathcal{T}|$) from the target domain at train time. We train the feature extractor network $f_\theta$, generative model $G_\phi$ and the classifier $\mathcal{C_\psi}$ on the source data and fine tune on the target samples. We compare our method against M-ADA \cite{qiao2020learning} and FADA \cite{motiian2017few} on the Digits dataset. The results are presented in Table \ref{table:digitsDA}. We outperform our nearest competitors by a significant margin, thus highlighting the adaptability of our method to this use-case.

\begin{table}[hbt!]
\centering
  \scalebox{0.69}{
  \begin{tabular}{lccccc}
    \toprule
        {Method} & {$|\mathcal{T}|$}& {U$\rightarrow$M}
        & {M$\rightarrow$S}
        & {S$\rightarrow$M} & {Avg.}
        \\
    \midrule
    
    
    
    FADA \cite{motiian2017few}&7&91.50&47.00&87.20&75.23\\
    CCSA \cite{motiian2017unified}&10&95.71&37.63&94.57&75.97\\
    \midrule
    \multirow{3}{*}{M-ADA \cite{qiao2020learning}}&0&71.19&36.61&60.14&55.98\\
    &7&92.33&56.33&89.90&79.52\\
    &10&93.67&57.16&91.81&80.88\\
    \midrule
    \multirow{3}{*}{Ours}&0&74.52&42.96&64.12&60.53\\
    &7&93.81&58.92&92.02&81.58\\
    &10&\textbf{96.10}&\textbf{60.07}&\textbf{95.33}&\textbf{83.83}\\
    \midrule
    
    
    
  \end{tabular}
  }
  \caption{Comparison of few-shot domain adaptation performance between different models on Digits dataset (MNIST \cite{lecun1998gradient} (M), USPS \cite{hull1994database} (U)    and SVHN \cite{netzer2011reading} (S)).}
  \label{table:digitsDA}
\end{table}
\vspace{-0.35cm}
  \subsubsection{Single Source Domain Generalization} \label{section:single_src_dg}
\vspace{-0.3cm}  
\begin{table}[hbt!]
\centering
  \scalebox{0.69}{
  \begin{tabular}{lccc}
    \toprule
        {Method} & {S$\,\to\,$P}
        & {S$\,\to\,$C}
        & {S$\,\to\,$A} 
        \\
        \midrule
        Deep All&29.88&32.47&30.56\\
        Ours (w/o $G_\phi$)&33.76&37.94&36.02\\
        Ours (w/o $f_\theta$)&50.39&66.82&44.94\\
        Ours ($f_\theta + G_\phi$)&53.82&70.33&	50.61\\
    
    \bottomrule
  \end{tabular}
  }
  \caption{Ablation on Single Source DG with Sketch (S) as source and Photo (P), Cartoon (C) and Art (A) as unseen target domains.}
  \label{table:single_source_ablation}
\end{table}

In single source DG, we only have access to a single domain during training and aim to generalize to all other unseen domains. We train on the Sketch domain of the PACS dataset and test on the other three domains (i.e. Photo, Art Painting and Cartoon). We compare our method against Jigen \cite{carlucci2019domain} and RSC \cite{huang2020self}. The results are presented in Figure \ref{fig:single_source_dg}. We also examine the individual effects of each of the components $f_\theta$ and $G_\phi$ of our method in this scenario by examining the performance without each of them (Table \ref{table:single_source_ablation}). We observe that without $G_\phi$, the performance difference between Deep All and our method is substantially lower than the improvement obtained by performing target projections on Deep All features, highlighting the effectiveness of the projection procedure. The best performance is obtained when both are used together, since the projection procedure effectively utilizes the label-preserving metric defined by $f_\theta$. 
\vspace{-0.16cm}
\subsection{Robust Domain Generalization}
\vspace{-0.15cm}
We examine the robustness of our method against different types of corruptions. We benchmark against the CIFAR-10-C dataset \cite{hendrycks2018benchmarking}, which consists of images with 19 corruptions types applied at five levels of severities on the test set of CIFAR-10. We follow the protocol detailed in \cite{qiao2020learning} and train our model on the CIFAR-10 dataset using a Wide Residual Network (WRN) backbone \cite{BMVC2016_87}. The results are presented severity level-wise in Table \ref{table:robust_cifar10C} and by type of corruption in Figure \ref{fig:robust_dg_cifar}. Similar to \ref{section:single_src_dg}, we observe the effectiveness of $G_\phi$ in generalizing to corruptions. 


\begin{table}[hbt!]
\centering
  \scalebox{0.69}{
  \begin{tabular}{lccccc}
    \toprule
        {Method} & {Level 1}
        & {Level 2}
        & {Level 3} & {Level 4} & {Level 5}
        \\
    \midrule
    ERM \cite{koltchinskii2011oracle}&87.8$\pm$0.1&81.5$\pm$0.2&75.5$\pm$0.4&68.2$\pm$0.6&56.1$\pm$0.8\\
    M-ADA \cite{qiao2020learning}&90.5$\pm$0.3&86.8$\pm$0.4&82.5$\pm$0.6&76.4$\pm$0.9&65.6$\pm$1.2\\
    Ours&\textbf{93.6$\pm$0.2}&\textbf{89.2$\pm$0.4}&\textbf{85.3$\pm$0.1}&\textbf{79.0$\pm$0.3}&\textbf{68.2$\pm$0.6}\\
    \bottomrule
  \end{tabular}
  }
  \caption{Comparison of performance on CIFAR-10-C with 5 different levels of severity. Accuracy averaged over all 19 corruption levels}
  \label{table:robust_cifar10C}
\end{table}

\begin{figure}[tb!]
\captionsetup{skip=-8pt}
\centering
\subfloat[][]{\includegraphics[height=3.85cm]{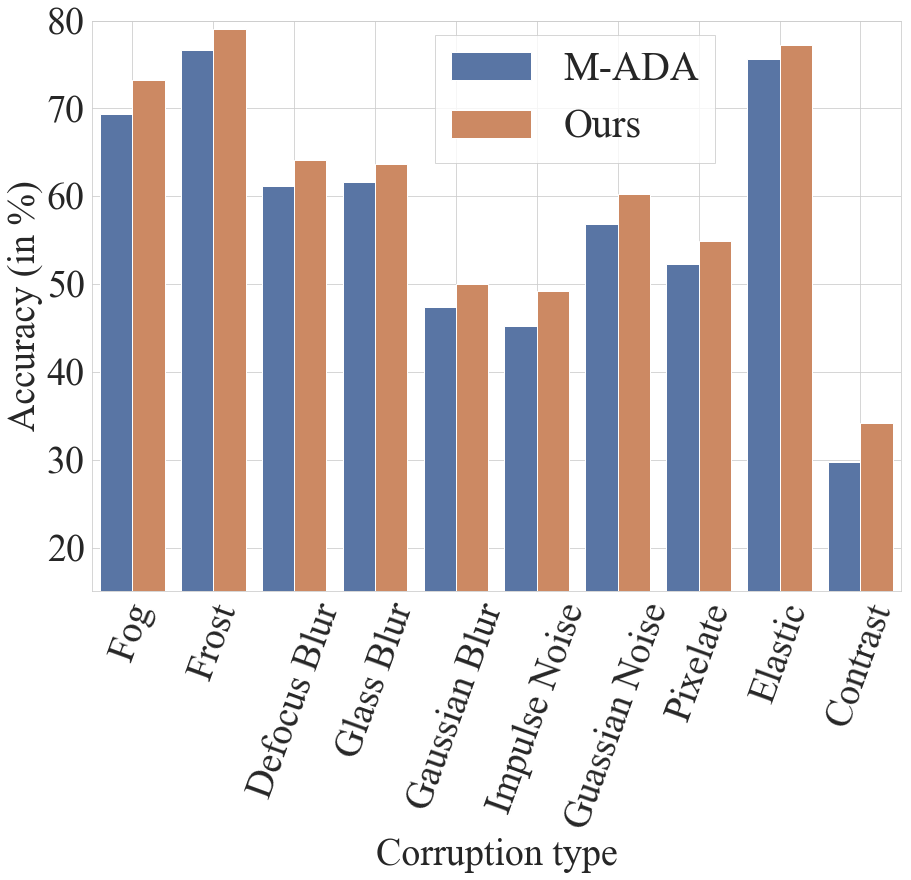} \label{fig:robust_dg_cifar}} 
\hspace{0.0cm}
\subfloat[][]{\includegraphics[height=3.85cm]{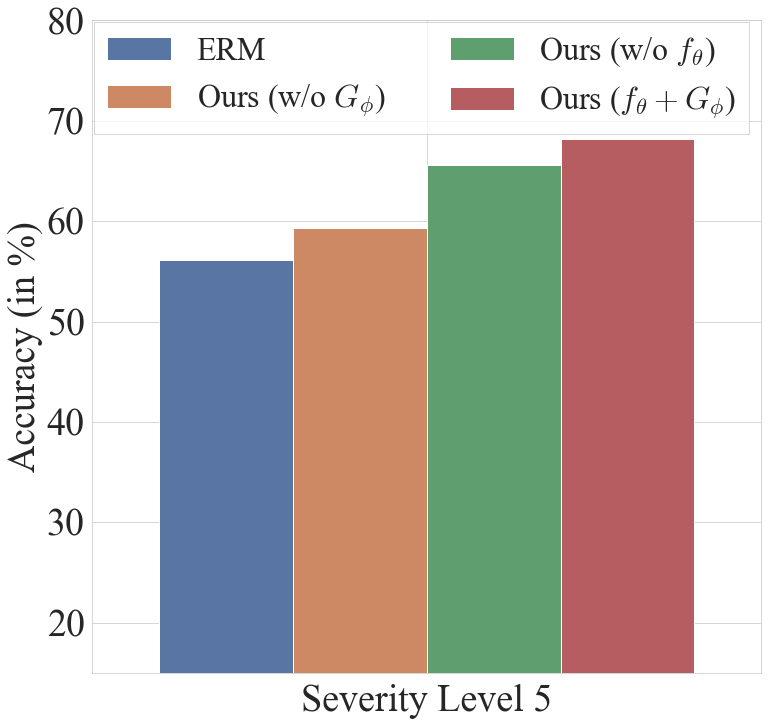}}
\caption{\textbf{(left)} \textbf{(a)} Performance comparison between M-ADA and our method on 10 out of the 19 corruption types at severity level 5 from the CIFAR-10-C dataset. \textbf{(right)} \textbf{(b)} Plot showing relative importance of each component of our method on the CIFAR-10-C dataset. Accuracy averaged over all 19 corruption levels.}

\end{figure}
\vspace{-0.2cm}
\subsection{Determining the stopping criteria} \label{section:stopping_criteria}
\vspace{-0.2cm}
The inference-time iterative optimization process discussed in \ref{section:analysis_and_implementation} needs a stopping criteria, since stopping too early would not guarantee the label preservation, while stopping too late may take the projected target to a low-probability region on the source manifold. 

We address this issue heuristically: stop the iteration process at the ``elbow-point" (maxima of the second derivative)  of the loss curve as a function of the number of iterations (denoted by $\text{n}^\ast$). This choice is inspired by the observation that the elbow-point reflects the point of diminishing returns; for a constant iteration rate ($\beta$), the loss decreases at a slower rate beyond this point. 

As empirical evidence, we vary the stopping point for each target sample by a fixed number of iterations $\epsilon$ around the maxima $\text{n}^\ast$ of the second derivative values of the sample and calculate the accuracy on the target projections so obtained. The plot of accuracy vs $\epsilon$ is shown in Figure \ref{fig:dev_nstar}. We observe that the highest accuracy is obtained around $\epsilon=0$ which indicates the correctness of the stopping criteria $\text{n}^\ast$. For negative values of $\epsilon$, the accuracy drop can be explained by inadequate label preservation, while the projected target leaves the source manifold for positive values of $\epsilon$.
\setlength{\textfloatsep}{1\baselineskip plus 0.2\baselineskip minus 0.5\baselineskip}
\vspace{-0.45cm}
\section{Conclusion}
\vspace{-0.24cm}

We propose a novel Domain Generalization technique where the source domains are utilized to learn a domain invariant label-preserving metric space. During inference, every target sample is projected onto this space so that the classifier trained on the source features can generalize well on the projected target sample. We have demonstrated that this method yields SOTA results on Multi Source, Single Source and Robust Domain Generalization settings. In addition, the data-efficiency of the method makes it suitable to work well in Low Resource settings. Future iterations of work could attempt to extend this method to Domain Generalization for Segmentation and Zero-Shot Learning. 

\appendix

\section{Theory}



\begin{prop}
\label{prop_delta}
The label-preserving transformation $f$ defined in Eq. 1 of the main paper reduces the $\mathcal{H}$-divergence between any pair of domains on which it is learned. Stated explicitly:

Assume the case of Binary Classification. We assume that class prior probabilities for each class is equal in all the source domains, i.e.
    \begin{equation}
        \mathbf{P}(\mathcal{Y}_i = \ell) = \mathbf{P}(\mathcal{Y}_k = \ell) \quad \forall i, k \in \{1, ..., |S|\}, \forall \ell \in \{0, 1\}.
     \end{equation}
     where $\mathbf{P}(\mathcal{Y}_i = {\ell}) = \mathop{\mathbb{E}}_{x \in \mathcal{D}_i^S} [\mathbbm{1}_{g(x)=\ell}]$ denotes the proportion of examples having label $\ell$ in source domain $\mathcal{D}^S_i$.
If there exists a metric space denoted by $(\mathcal{F},d)$ and a transformation function $f : \mathcal{X} \to \mathcal{F}$ such that $d(f(x_1), f(x_2)) = 0  \iff g(x_1)=g(x_2)$, i.e. $x_1$ and $x_2$ have the same labels (irrespective of domain), then,
 \begin{equation}
 \Delta_{\mathcal{H'}}[f(\mathcal{D}_i^S), f(\mathcal{D}_k^S)] = 0  \quad \forall i, k \in \{1, ..., |S|\}
\end{equation} where $\mathcal{H'}$ denotes the space of hypotheses $h' : \mathcal{F} \to \mathcal{Y}$, and $f(\mathcal{D})$ denotes the distribution $\mathcal{D}$ under the transformation $f$.

\end{prop}

\begin{proof}
For any $i,k\in[|S|]$ and any $h' \in \mathcal{H'}$,
\begin{align}
    \left| \mathbf{P}_{x \in \mathcal{D}_i^S}[h'(f(x)) = 1]- \mathbf{P}_{x \in \mathcal{D}_k^S}[h'(f(x)) = 1]\right|
\end{align}
\begin{equation}
     =\left| \mathop{\mathbb{E}}_{x \in \mathcal{D}_i^S} [\mathbbm{1}_{h'(f(x)) = 1}] - \mathop{\mathbb{E}}_{x \in \mathcal{D}_k^S} [\mathbbm{1}_{h'(f(x)) = 1}] \right|
\end{equation}
\newline
\newline
Now, note that $g(x_1)=g(x_2)\implies d(f(x_1), f(x_2)) = 0 \implies f(x_1)=f(x_2)$
 $\implies h'(f(x_1))=h'(f(x_2))$. 
\newline 
This means that all the examples of a given class are assigned the same label $\ell$ by $h'$, regardless of the source domain. Depending on the labels assigned to the different classes by $h'$, these three cases would arise:
\newline
\newline
\textbf{Case 1:} $h'(f(x)) = 1 \quad \forall x \in \mathcal{X}$
\begin{equation*}
    \implies \mathop{\mathbb{E}}_{x \in \mathcal{D}_i^S} [\mathbbm{1}_{h'(f(x)) = 1}] = \int_{\mathcal{X}} \mathcal{D}_i^S(x) dx = 1 
\end{equation*}
\begin{equation}
    = \int_{\mathcal{X}} \mathcal{D}_k^S(x) dx = \mathop{\mathbb{E}}_{x \in \mathcal{D}_k^S} [\mathbbm{1}_{h'(f(x)) = 1}]
\end{equation}
\begin{equation*}
    \therefore \left|\mathbf{P}_{x \in \mathcal{D}_i^S}[h'(f(x)) = 1]- \mathbf{P}_{x \in \mathcal{D}_k^S}[h'(f(x)) = 1]\right| = 0    
\end{equation*}
\newline
\textbf{Case 2:} $h'(f(x)) = 0 \quad \forall x \in \mathcal{X}$
\begin{equation}
    \mathop{\mathbb{E}}_{x \in \mathcal{D}_i^S} [\mathbbm{1}_{h'(f(x)) = 1}] = \mathop{\mathbb{E}}_{x \in \mathcal{D}_j^S}[\mathbbm{1}_{h'(f(x)) = 1}] = 0    
    \end{equation}
\begin{equation*}
    \therefore \left|\mathbf{P}_{x \in \mathcal{D}_i^S}[h'(f(x)) = 1]- \mathbf{P}_{x \in \mathcal{D}_k^S}[h'(f(x)) = 1]\right| = 0
\end{equation*}
\newline
\textbf{Case 3:} $h'(f(x)) = 1 \iff  g(x)=\ell$ for some $\ell \in \{0,1\}$
\begin{equation}
    \mathop{\mathbb{E}}_{x \in \mathcal{D}_i^S} [\mathbbm{1}_{h'(f(x)) = 1}] = \mathop{\mathbb{E}}_{x \in \mathcal{D}_i^S} [\mathbbm{1}_{g(x)=\ell}]  = \mathbf{P}(\mathcal{Y}_i = {\ell})    
\end{equation}

Similarly, 
\begin{equation}
\mathop{\mathbb{E}}_{x \in \mathcal{D}_k^S} [\mathbbm{1}_{h'(f(x)) = 1}] = \mathbf{P}(\mathcal{Y}_k = {\ell})
\end{equation}
\begin{multline}
    \therefore\left|\mathbf{P}_{x \in \mathcal{D}_i^S}[h'(f(x)) = 1]- \mathbf{P}_{x \in \mathcal{D}_k^S}[h'(f(x)) = 1]\right| \\
 = \left| \mathbf{P}(\mathcal{Y}_i ={\ell}) - \mathbf{P}(\mathcal{Y}_k = {\ell}) \right| 
 = 0  
\end{multline}

Since, 
\begin{equation}
\left|\mathbf{P}_{x \in \mathcal{D}_i^S}[h'(f(x)) = 1]- \mathbf{P}_{x \in \mathcal{D}_k^S}[h'(f(x)) = 1]\right|=0 \ \forall h' \in \mathcal{H'}
\end{equation}
we have,
\begin{multline}
\Delta_\mathcal{H'}[f(\mathcal{D}_i^S), f(\mathcal{D}_k^S) ] = \\\sup_{h' \in \mathcal{H'}}\left|\mathbf{P}_{x \in \mathcal{D}_i^S}[h'(f(x)) = 1]- \mathbf{P}_{x \in \mathcal{D}_k^S}[h'(f(x)) = 1]\right| = 0 \\ \forall i, k \in \{1, ..., |S|\}
\end{multline}
\newline
\end{proof}

We note here that in practice, the assumption $g(x_1)=g(x_2) \implies d(f(x_1), f(x_2)) = 0$ need not necessarily hold, i.e. two different images having the same label might not have coincident representations. However, the optimization procedure detailed in Eq. 1 of the main paper forces representations of the same class together into the same cluster, as described in Section \ref{section:class_alignment}, thus resulting in a higher similarity score for images with the same label.

\begin{prop}
The expected misclassification rate obtained with a classifier $h'$ when the projected target $\mathbf{z_t^\ast}$ is used instead of the true target $\mathbf{z_t}$, obeys the following upper-bound:
\begin{align}
&\mathbb{E}_{(\mathcal{D}^T, \mathcal{D}^{T^\ast})} \left| \tilde{g}(\mathbf{z_t}) - h'(\mathbf{z_t^*}) \right| \leq \nonumber \\
    & \underbrace{\mathbb{E}_{\mathcal{D}^{T^\ast}} \left| \tilde{g}(\mathbf{z_t^*}) - h'(\mathbf{z_t^*}) \right|}_\text{i} + \underbrace{\mathbb{E}_{(\mathcal{D}^T, \mathcal{D}^{T^\ast})} \left| \tilde{g}(\mathbf{z_t}) - \tilde{g}(\mathbf{z_t^*}) \right|}_\text{ii}
    \label{eqn:risk_ineq}
\end{align}
where $\mathcal{D}^T$ and $\mathcal{D}^{T^\ast}$ respectively denote the true and the projected target distributions, respectively. 
\end{prop}
\begin{proof}

Using the triangle inequality on $\mathbb{R}$, we get
\begin{align}
\left| \tilde{g}(\mathbf{z_t}) - h'(\mathbf{z_t^*}) \right| \leq \left|\tilde{g}(\mathbf{z_t^*}) - h'(\mathbf{z_t^*}) \right| +\left| \tilde{g}(\mathbf{z_t}) - \tilde{g}(\mathbf{z_t^*}) \right|
    \label{eqn:risk_ineq1}
\end{align}
$\mathbf{z_t}$ and $\mathbf{z_t^\ast}$ are random variables denoting feature vectors corresponding to the target point and the point obtained by projecting the target point to the source manifold. We assume that these come from the true 
($\mathcal{D}^T$)
and projected ($\mathcal{D}^{T^\ast}$) target distributions, respectively. 
To obtain the average misclassification rate, we take an expectation with respect to their joint distribution.

Since expectation is a linear operation, we can take expectation over both sides of the inequality. Doing so with respect to the joint distribution of $(\mathcal{D}^T, \mathcal{D}^{T^\ast})$ 
on both sides,
\begin{align}
&\mathbb{E}_{(\mathcal{D}^T, \mathcal{D}^{T^\ast})} \left| \tilde{g}(\mathbf{z_t}) - h'(\mathbf{z_t^*}) \right| \leq \nonumber \\
    & {\mathbb{E}_{\mathcal{D}^{T^\ast}} \left| \tilde{g}(\mathbf{z_t^*}) - h'(\mathbf{z_t^*}) \right|} + {\mathbb{E}_{(\mathcal{D}^T, \mathcal{D}^{T^\ast})} \left| \tilde{g}(\mathbf{z_t}) - \tilde{g}(\mathbf{z_t^*}) \right|}
    \label{eqn:risk_ineq}
\end{align}
Note that in term (i) of the inequality, we can omit the expectation over $\mathcal{D}^T$ since there is no dependence on the random variable $\mathbf{z_t}$.

\end{proof}

\begin{figure}
\includegraphics[width=.44\textwidth,height=.47\textwidth]{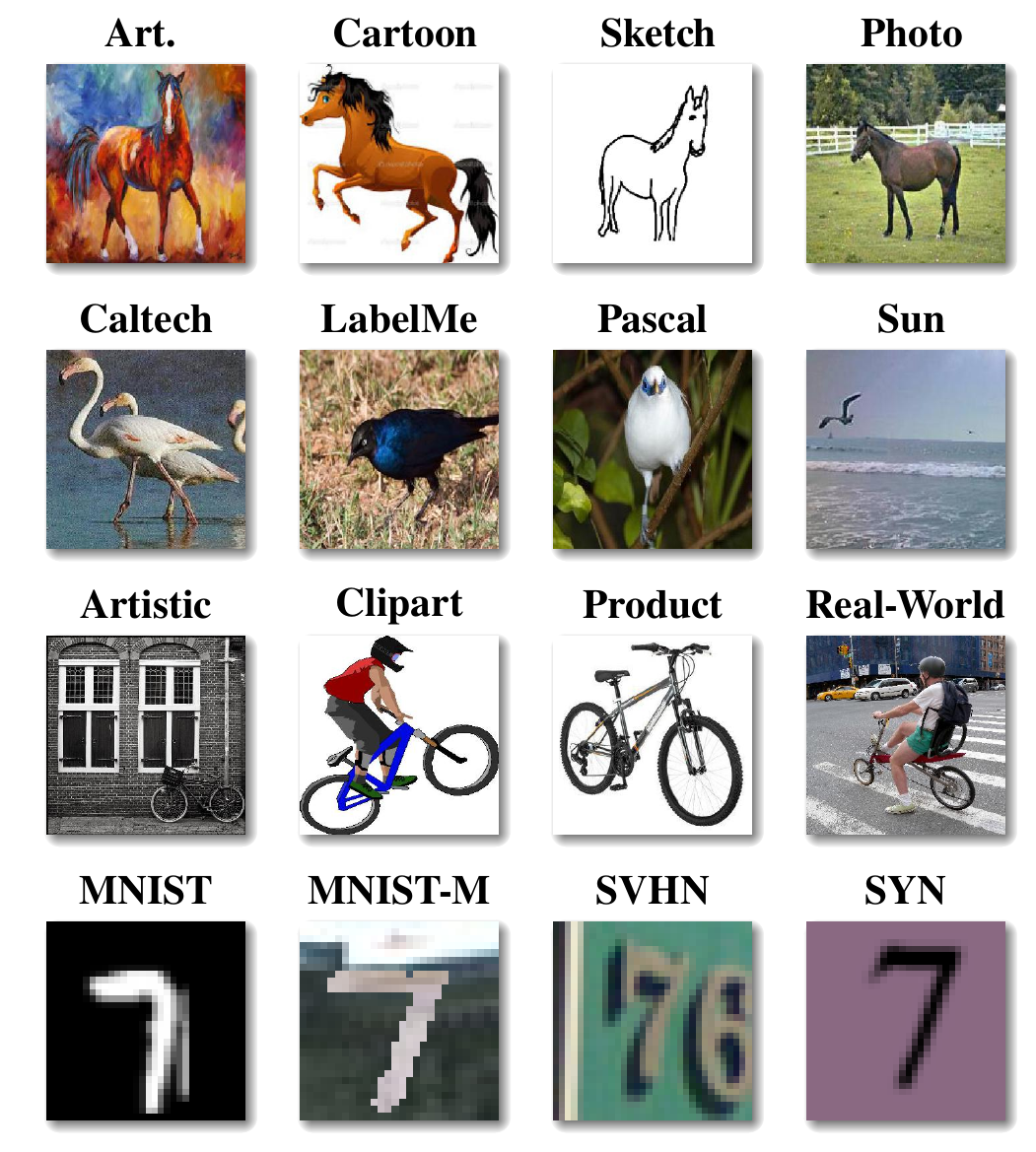}
\caption{Few example images from PACS (1st row), VLCS (2nd row), Office-Home (3rd row) and Digits-DG (4th row) datasets.}
\label{fig:dg_samples}
\end{figure}

\begin{figure}
\includegraphics[width=.475\textwidth,height=.45\textwidth]{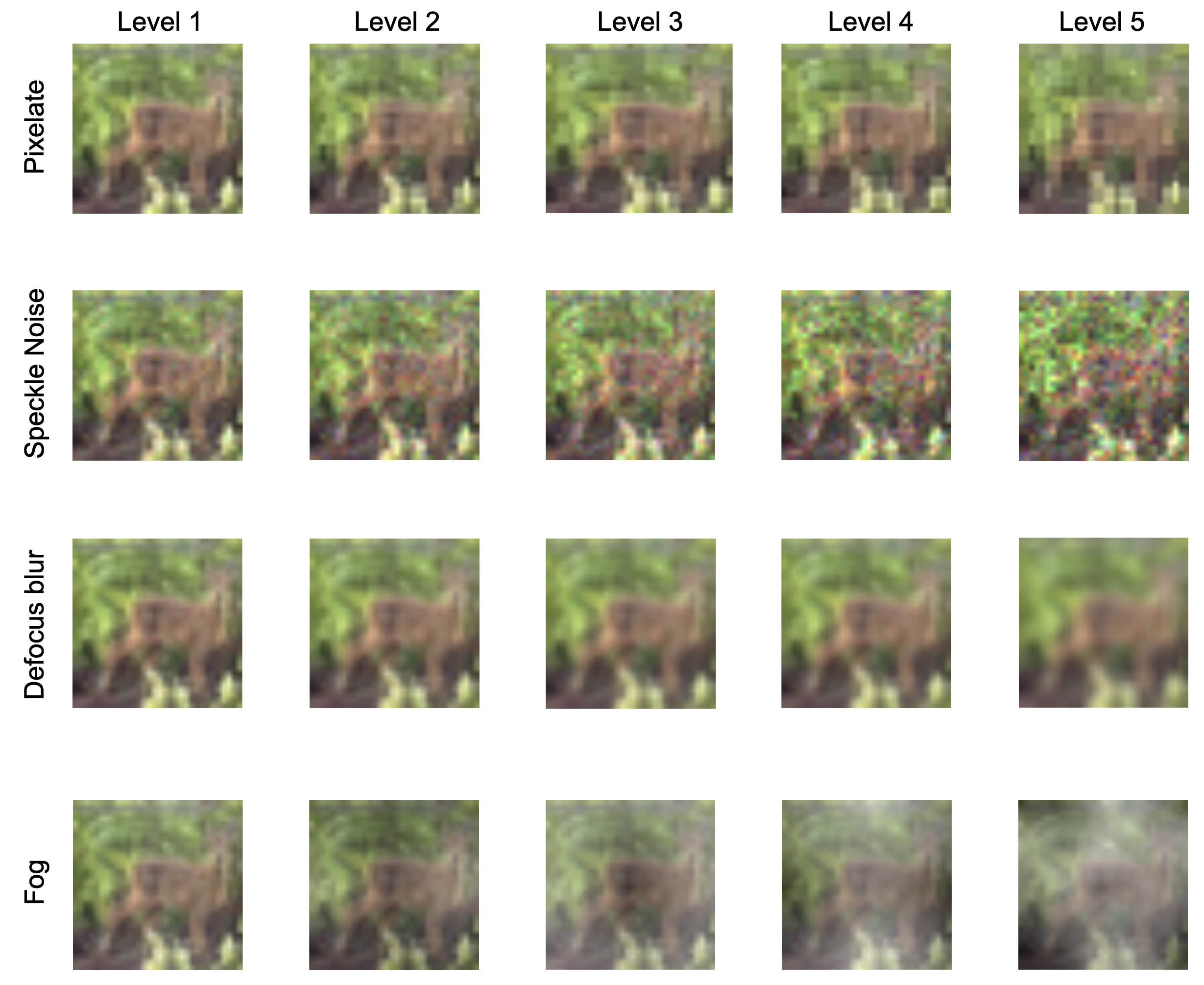}
\caption{Few example images from CIFAR-10-C dataset depicting 5 levels of severities and four of the 19 Corruption types.}
\label{fig:robust_dg_samples}
\end{figure}

   
   

\section{Datasets}
We demonstrate the effectiveness of our algorithm on the following DG datasets:

\textbf{PACS:} Stands for \textbf{P}hoto, \textbf{A}rt Painting, \textbf{C}artoon and \textbf{S}ketch. This dataset contains a total of 9991 images taken from different sources such as Caltech256, Sketchy, TU-Berlin and Google Images. Each image is assigned one out of seven possible labels namely dog, elephant, giraffe, guitar, horse, house or person. The substantial domain shift in the dataset poses a significant challenge to DG methods. 

\textbf{VLCS:} Stands for \textbf{V}OC2007(Pascal), \textbf{L}abelMe, \textbf{C}altech and \textbf{S}un where each one of them is a different dataset differing in the camera specifications. There are a total of 10729 photos in the whole dataset where each photo is assigned one out of five labels namely bird, car, chair, dog, or person. 

\textbf{Office-Home:} This dataset is comprised of four domains namely Art, Clipart, Product and Real-World. There are a total of 15588 images in the dataset and each image is assigned one out of 65 classes. The Real-World images are taken with a regular camera and the Product images are taken from a vendor websites and thus differ in background and quality.

\textbf{Digits-DG:} This digit recognition (0-9) dataset contains a mixture of 4 different domains namely MNIST, MNIST-M, SVHN, and SYN that differ drastically in font style and background. The dataset contains 24000 images.

\textbf{CIFAR-10-C:} This dataset is a robustness benchmark for deep learning systems. It consists of test images of CIFAR-10 with 19 different types of corruption having 5 different severity levels (1-5) with 5 indicating most severe. The models are trained on CIFAR-10 and evaluated on CIFAR-10-C. Please refer to Figure \ref{fig:robust_dg_samples} for a few samples from different severity levels.

\begin{figure*}[hbt!]
\centering
\subfloat[][Photo as source]{\includegraphics[width=0.250\textwidth]{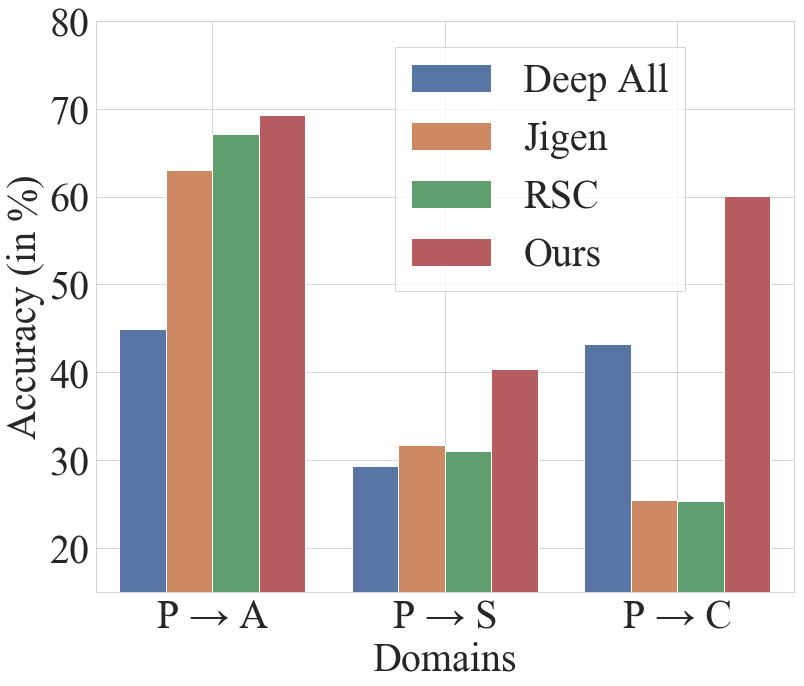}\label{fig:h_div}}
\hspace{0.94cm}
\subfloat[][Cartoon as source]{\includegraphics[width=0.250\textwidth]{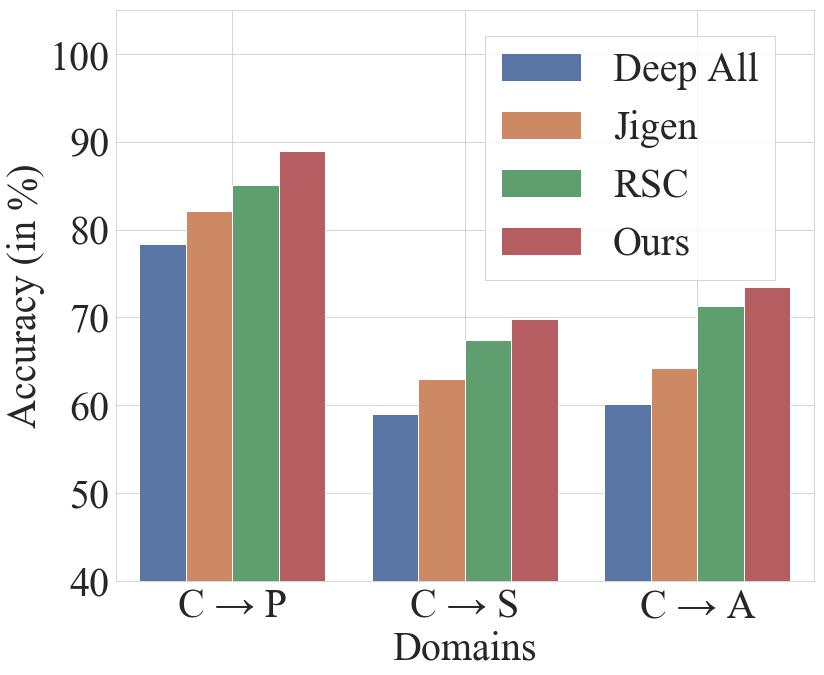}\label{fig:ablation}}
\hspace{0.94cm}
\subfloat[][Art as source]{\includegraphics[width=0.250\textwidth]{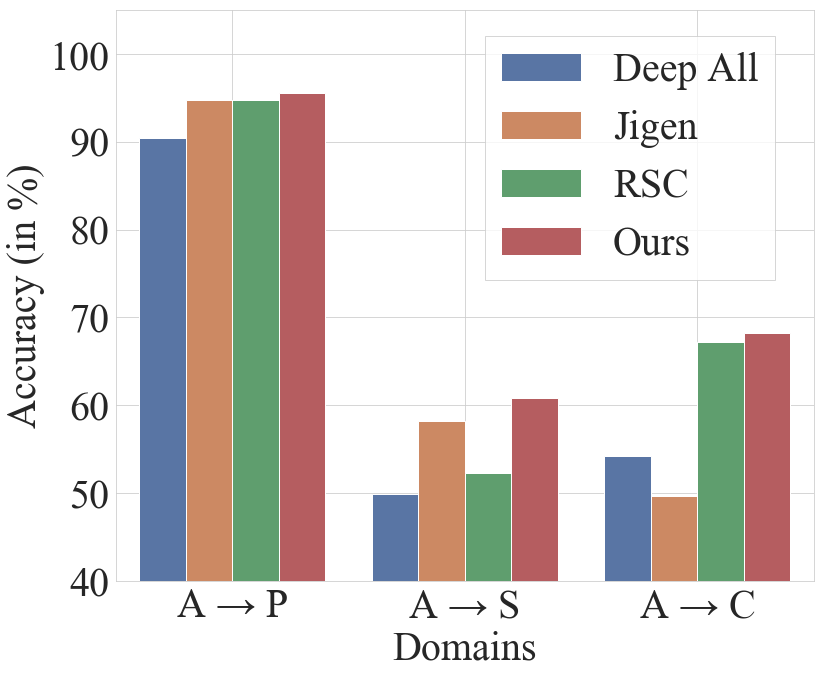}}
\caption{
Relative performance of DG methods trained on a single domain of PACS dataset.}
\label{fig:scarce_data}
\end{figure*}

\setlength{\textfloatsep}{1\baselineskip plus 0.2\baselineskip minus 0.5\baselineskip}
\setlength{\textfloatsep}{1\baselineskip plus 0.2\baselineskip minus 0.5\baselineskip}


\setlength{\textfloatsep}{1\baselineskip plus 0.2\baselineskip minus 0.5\baselineskip}
\section{Implementation Details}
\subsection{System Configuration and Frameworks}
All the experiments are performed using an Intel Xeon processor (12 Cores) with a base frequency of 2.7 GHz, 32GB RAM, Ubuntu OS and four NVIDIA® Tesla® V100 (16 GB Memory) GPUs. Our implementation was written in Python 3.6.10 and uses PyTorch version 1.6.0 running on CUDA version 10.1. 
   
\setlength{\textfloatsep}{1\baselineskip plus 0.2\baselineskip minus 0.5\baselineskip}
\subsection{Architectural details}
Please refer to Figure \ref{fig:architectures} for different architectures and backbones that we employed to evaluate the performance of our method on each dataset.

We have compared the number of parameters in other state-of-the-art (SoTA) models with ResNet-18 as the backbone: RSC - 11.18M, Jigen - 11.19M, EisNet - 23.5M, MMLD - 12.75M, DDAIG - 12.18M. For the proposed method, the number of parameters are 11.32M. This shows that it is comparable to the SoTA, since the classifier and the generator networks used are shallow. 
\subsection{Hyperparameter choices}\label{subsec:hparams}
We select hyperparameters based on model performance on a validation set consisting of data from the source domains. We use the splits given with the dataset whenever they are available. When the validation split for the source domains are not explicitly provided, we split our training data into a small validation set for model selection. We choose the best model based on (a) the average loss $\mathcal{L_A}$ (for $f_\theta$) (b) the average accuracy (for $\mathcal{C_\psi}$) (c) the average reconstruction/discriminator loss (for $G_\phi$) on the validation set.

\paragraph{Training the network $f_\theta$:}
We use the SGD optimizer with a learning rate of 0.001 to train the $f_\theta$ network  for all the datasets except for Digits-DG and CIFAR-10-C datasets. Digits-DG dataset is trained with a learning rate of 0.05 while for CIFAR-10-C, we train the network with an initial learning rate of 0.1.
We use the same backbone as described in \cite{zhou2020deep} for Digits-DG dataset and for CIFAR-10-C, the backbone employed is Wide Residual Network \cite{BMVC2016_87} (WRN) with depth and width of 16 and 4, respectively. We train for 200 epochs with the ResNet-18/ResNet-50 backbones, while AlexNet models are trained for 250 epochs. For the Digits-DG dataset, the $f_\theta$ network is trained for 100 epochs while WRN is trained for 200 epochs with a weight decay of 0.0005. In WRN, the initial learning rate of 0.1 is reduced to 0.02, 0.004 and 0.0008 at the 50th, 100th, and 150th epoch, respectively. As a regularization, we apply random image augmentations, namely horizontal flip (with probability 0.5), color jitter (with probability 0.8), greyscaling (with probability 0.2), applying Gaussian Blur (with a kernel of size 21) and adding Gaussian noise (with $\sigma = 0.2)$. Batch size for all the backbones is fixed to be 128. 

\paragraph{Training the VAE/GAN $G_\phi$:} 
VAE is trained with a learning rate of 0.005 and momentum of 0.9 using SGD optimizer for 350 epochs in all cases except for the Digits-DG dataset, where it is trained for 150 epochs. We train with a batch size of 64 in all settings. 
We use the standard VAE objective, with a combination of both the L1 and L2 losses for the reconstruction error term.
In case of GAN (on VLCS dataset), the learning rate is set to be 0.0002 and it is trained for 450 epochs.

\paragraph{Training the Classifier $\mathcal{C}_\psi$:}
The classifier $\mathcal{C}_\psi$ is trained using the Adam optimizer with a learning rate of 0.003 for 15, 20 and 30 epochs on the Digits-DG backbone, ResNet-based models and AlexNet-based models respectively, with a batch size of 64 in all cases.

\subsection{Iteration rate $\beta$}

\begin{table}[hbt!]
\centering
  \scalebox{0.95}{
  \begin{tabular}{lccc}
    \toprule
        \thead{Iteration \\rate ($\beta$)} &  \thead{Avg. number of \\iterations to stop}
        &  \thead{Total \\iterations ($M$)} & {Acc. (in \%)}
       
        \\
    \midrule
    0.05&465&20000&81.32\\
    0.01&602&20000&81.79\\
    0.005&2672&20000&80.92\\
    0.001&15976&20000&81.56\\

    \bottomrule
  \end{tabular}
  }
  \caption{Comparison of performance on Sketch as a target domain for different values of iteration rate $\beta$.}
  \label{table:beta}
\end{table}

To project a target example $\mathbf{z_t}$ on to the manifold of source features ($\mathcal{Z}_s$) during inference, we follow an iterative optimization procedure where a latent vector $\mathbf{u}$ (initially drawn from an isotropic Gaussian distribution) is optimized in the generative latent space of $G_\phi$. We examine the effect of varying the iteration rate $\beta$ in this optimization procedure through the following experiment: for a fixed iteration rate $\beta$, selected from the values $\{0.05, 0.01, 0.005, 0.001\}$, we perform the optimization process until the `optimal stopping point' (which corresponds to the ``elbow" point) is reached. We report the test accuracy in each case. As shown in Table \ref{table:beta}, we have validated that for a large number of iterations ($M=20000$), all reasonably small iteration rates give consistent performance on target domains. The second column in Table \ref{table:beta} reports the average of all optimal stopping points ($\text{n}^\ast$) of the target examples for the Sketch domain of the PACS dataset. 

Thus, we conclude that for reasonably small choice of $\beta$, \textit{the performance of the inference-time optimization procedure is unaffected by varying $\beta$}. Hence, we fix the iteration rate to be 0.01 in all our experiments for uniformity.


    


\begin{figure}[tb!]
\centering
\subfloat[][w/o $f_{\theta}$ (Deep All)]{\includegraphics[width=0.22\textwidth]{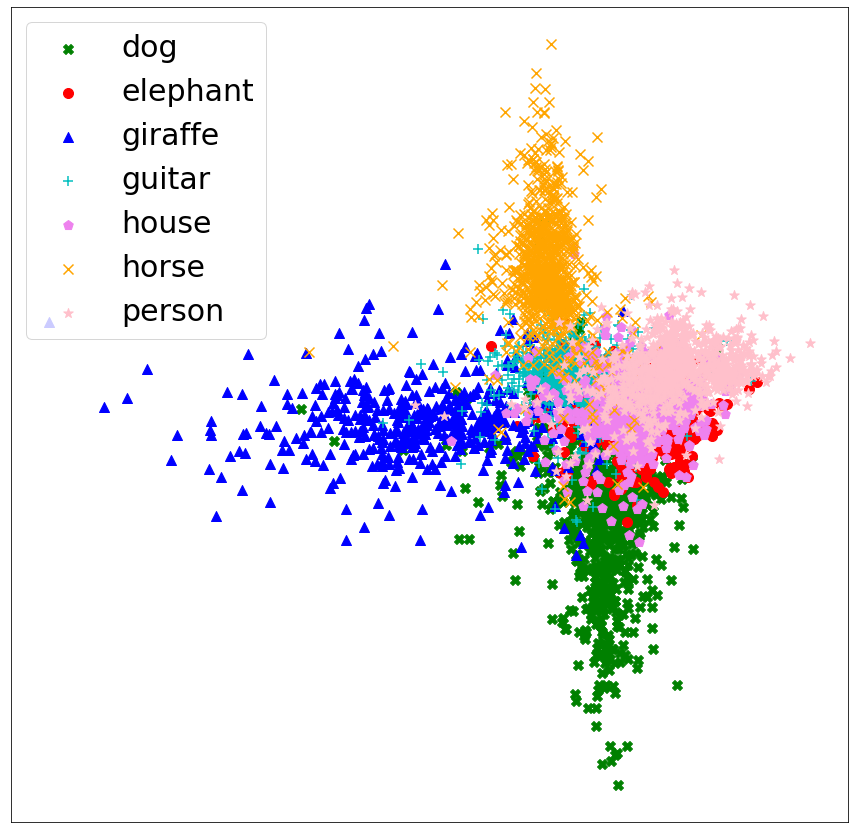}\label{fig:tsne_deepall}}
\hspace{0.2cm}
\subfloat[][with $f_{\theta}$]{\includegraphics[width=0.22\textwidth]{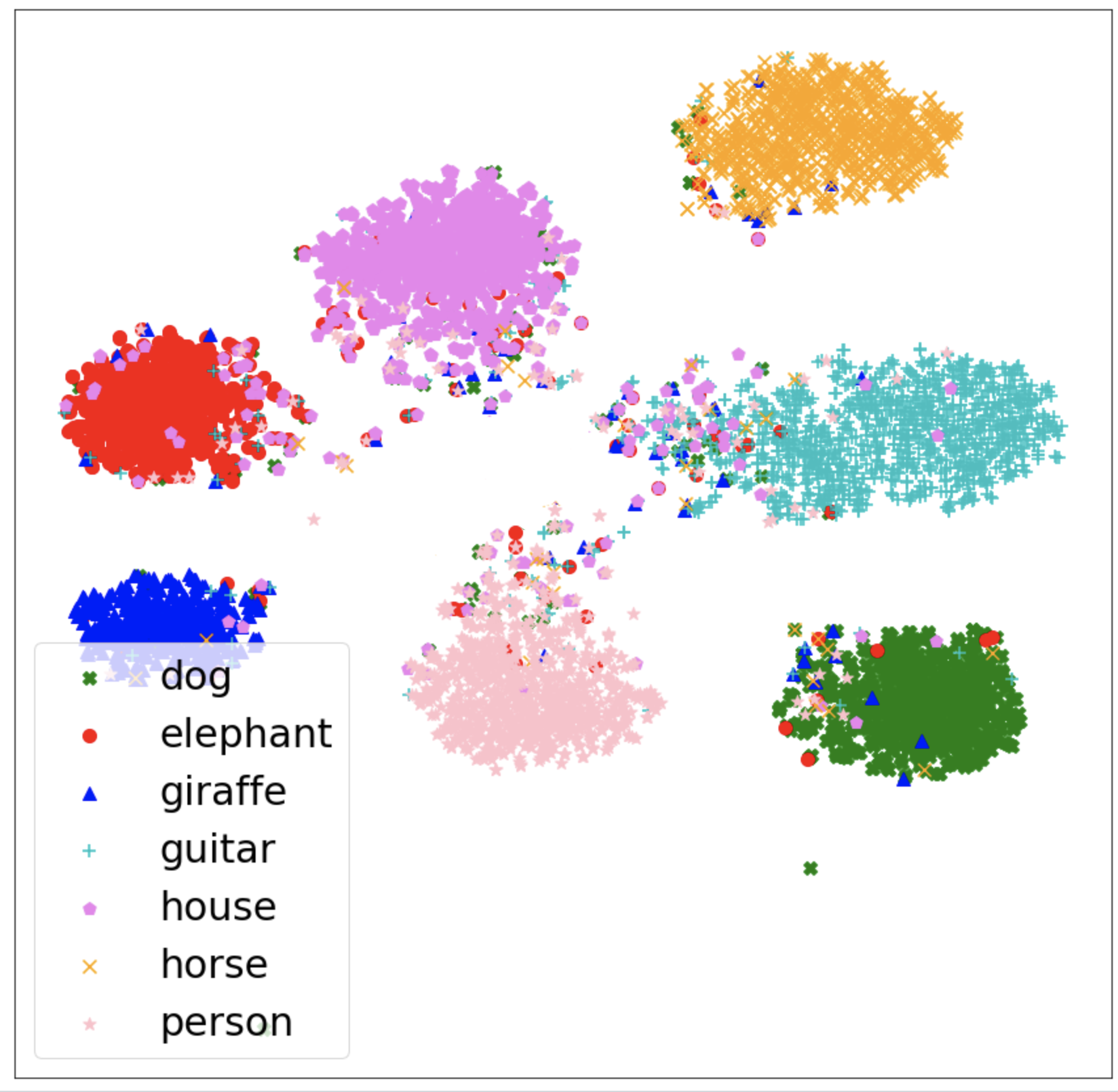}\label{fig:tsne_ftheta}}
\caption{T-SNE plot of features obtained from mixture of PAC (Photo, Art, Cartoon) as source domains using \textbf{a)} Deep All and \textbf{b)} $f_\theta$ network.}
\label{fig:tsne}
\end{figure}

\begin{figure}[tb!]
\centering
\subfloat[][source-source $\mathcal{A}$-distance]{\includegraphics[width=0.22\textwidth]{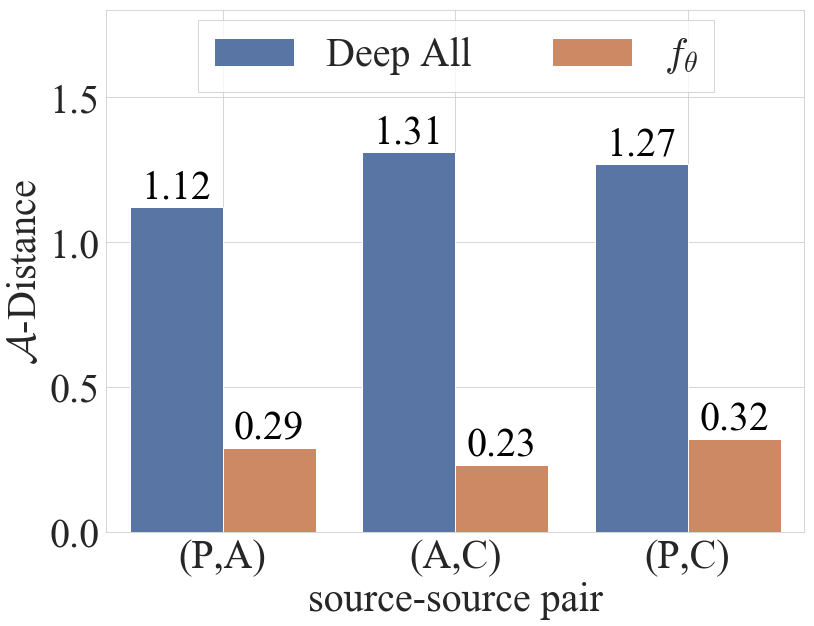}\label{fig:src-src_a_dist}}
\hspace{0.2cm}
\subfloat[][ablation with Digits-DG]{\includegraphics[width=0.22\textwidth]{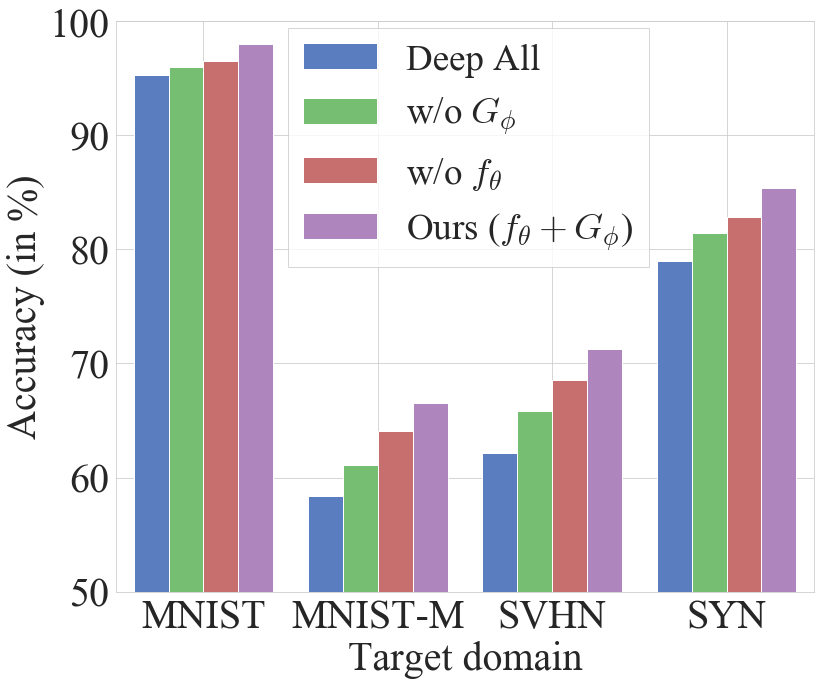}\label{fig:ablation_digits}}
\caption{\textbf{a)} $\mathcal{A}$-distance (lower is better) between sources using Deep All and  label-preserving features from $f_\theta$. \textbf{b)} Ablation on different components of the proposed method with Digits-DG dataset.}
\label{fig:add_hdiv_ablation_d_fg}
\end{figure}


\section{Additional Results}
\subsection{Class-wise alignment} \label{section:class_alignment}
Despite the large inter-domain discrepancy in the PACS dataset, it is observed that when we train the $f_\theta$ network on Photo, Art and Cartoon as source domains, we obtain features that align (or cluster) class-wise which is not the case with Deep All features on same source domains as shown by T-SNE plots in Figure \ref{fig:tsne}. This demonstrates the effectiveness of the objective of $f_\theta$ network in bringing examples from the same class together in the feature space $\mathcal{F}$ irrespective of the domains they belong to. It helps shallow classifiers (with a single fully connected layer as in $\mathcal{C}_\psi$) to distinguish between features which are more distinct across the class-labels. For features extracted from the Deep All model, the separation across the classes is less which makes the shallow classifier less robust and more error-prone.

We also demonstrate the effectiveness of the $f_\theta$ network in reducing source-source $\mathcal{H}$-divergence by comparing the $\mathcal{A}$-distance between features extracted from different source domains through Deep All and the $f_\theta$ network. These results are presented in \ref{fig:src-src_a_dist}.

\subsection{Low resource settings}
We train different models on each of the source domains, namely Photo, Cartoon and Art and test on the other three  domains of the PACS dataset. We observe that our method is significantly more data efficient compared to the baselines, as shown in Figure \ref{fig:scarce_data}, owing to the fact that $f_\theta$ learns label information from pairs of images and $G_\phi$ facilitates projection on the manifold of label-reserving source features.

\subsection{Ablation with Digits-DG}

We conduct ablation with Digits-DG \cite{zhou2020deep} dataset by training the proposed method with and without $f_\theta$ and $G_\phi$ networks. Figure \ref{fig:ablation_digits} shows the merit of the proposed components in improving the performance on the Digits-DG dataset.

\subsection{Robust DG}

Figure \ref{fig:corr_19} shows the robustness of our method against all 19 types of corruptions from the CIFAR-10-C dataset \cite{hendrycks2018benchmarking}. The proposed method trained on CIFAR-10 dataset, achieves generalization to the corrupted target images having severity level 5. 

\begin{figure}
\centering
  \scalebox{1}{
  \includegraphics[width=0.468\textwidth]{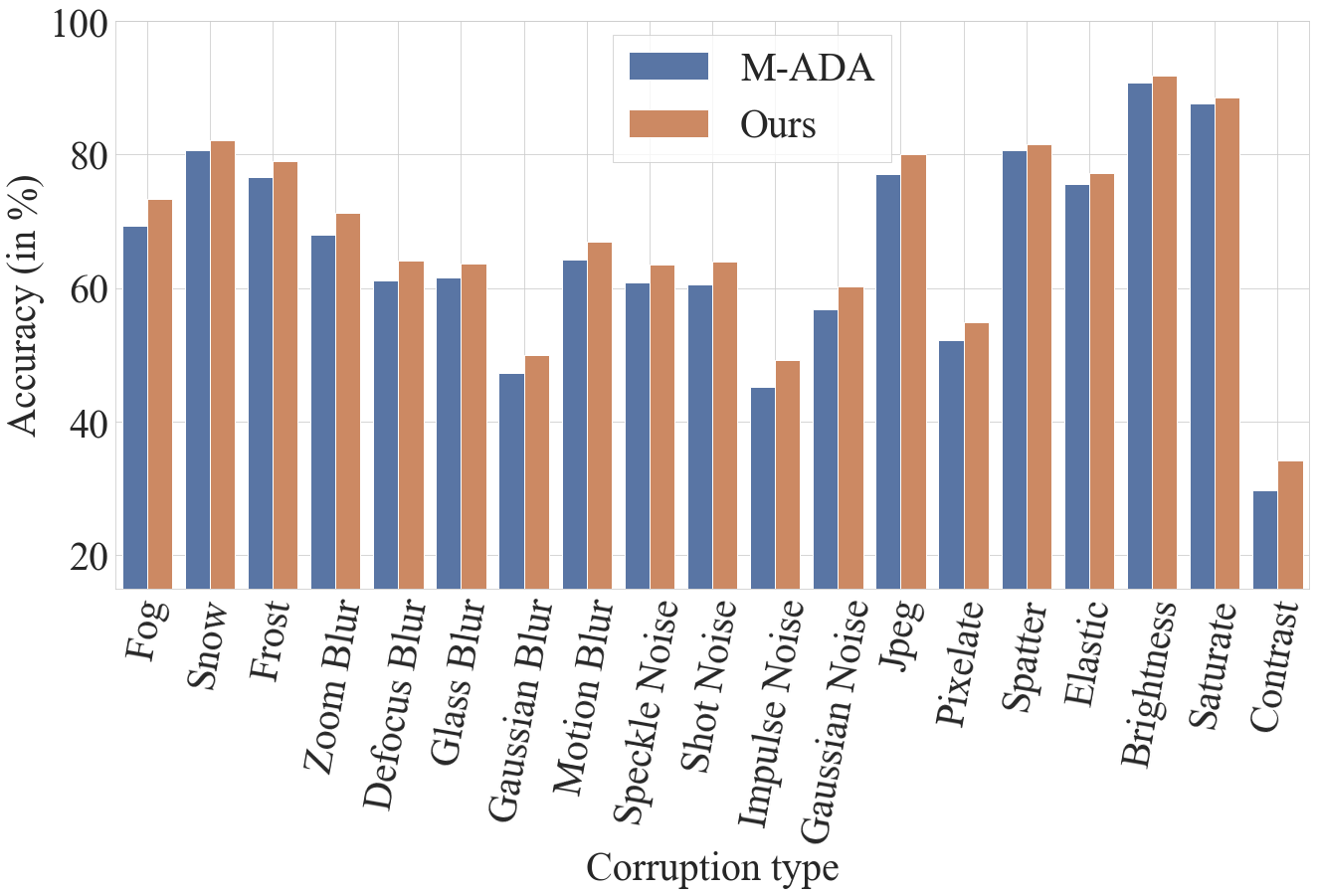}
  }
  \caption {Comparison of performance on all 19 Corruption types with severity level 5.}
  \label{fig:corr_19}
\end{figure}

\begin{figure*}
\centering
\subfloat[Multi-source DG on PACS with AlexNet]{
  \includegraphics[width=120mm]{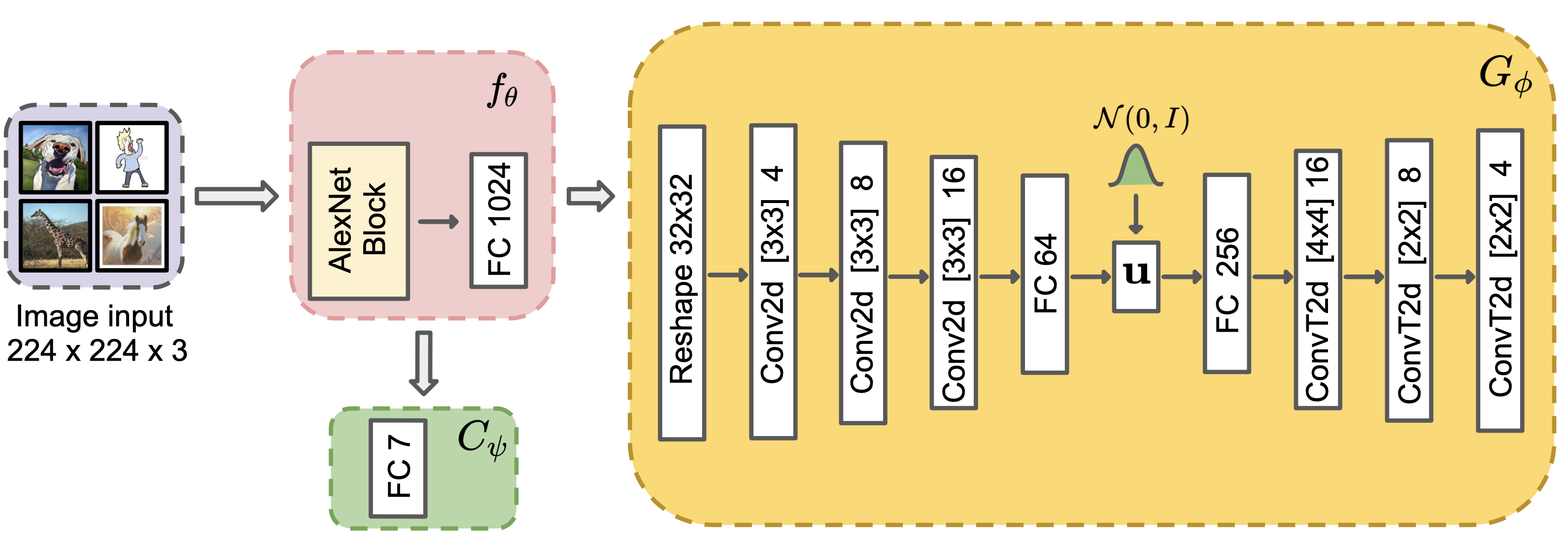}
}
\\
\vspace{0.9cm}
\subfloat[ Multi-source DG on PACS with RestNet-18]{
  \includegraphics[width=120mm]{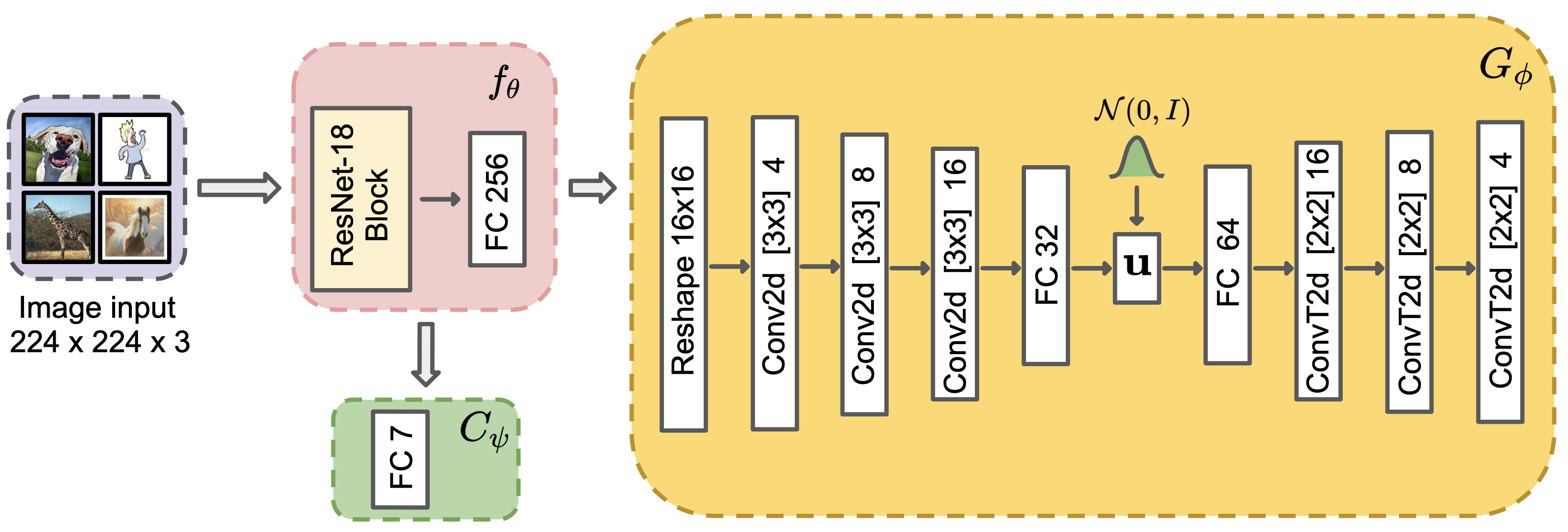}
}
\\
\vspace{0.9cm}
\subfloat[Multi-source DG on VLCS with AlexNet]{
  \includegraphics[width=120mm]{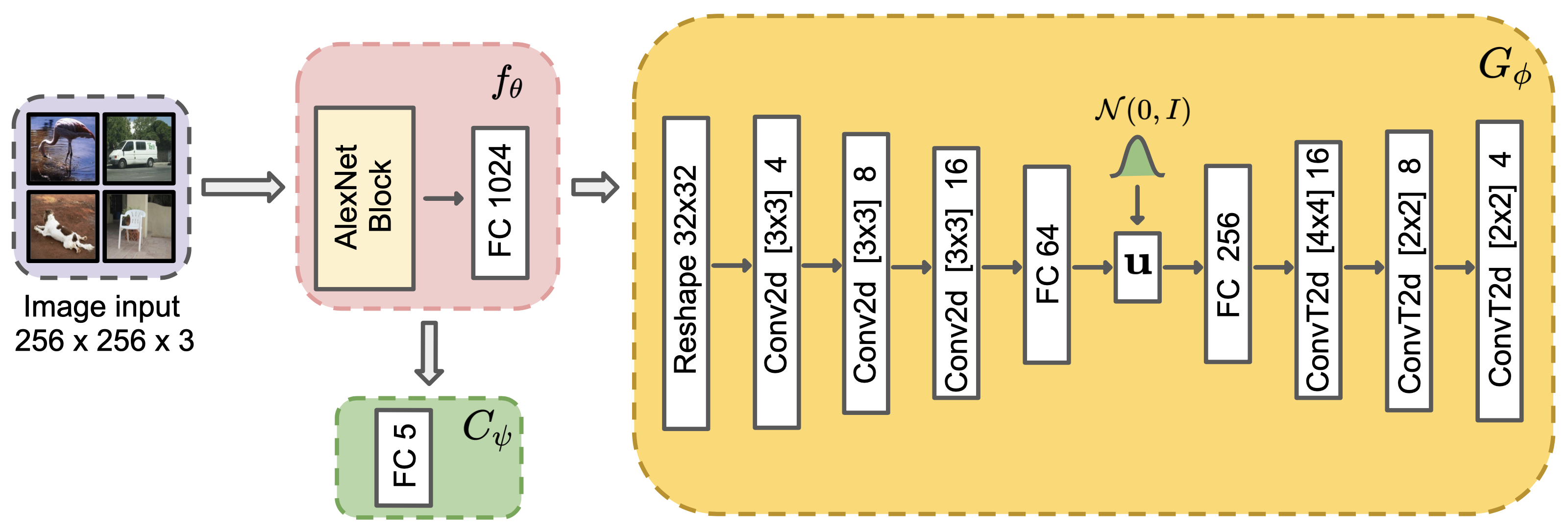}
}
\\
\vspace{0.9cm}
\subfloat[Multi-source DG on Office-Home with ResNet-18]{
  \includegraphics[width=120mm]{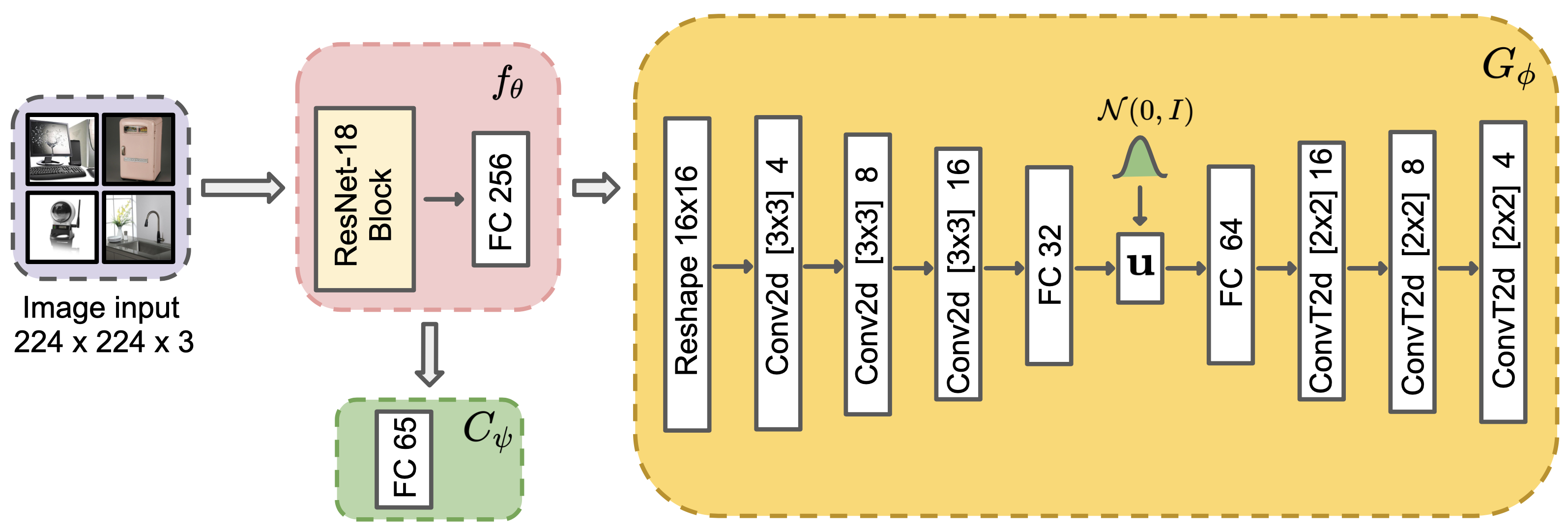}
}
\end{figure*}
\begin{figure*}
\ContinuedFloat
\centering
\subfloat[Multi-source DG on Digits-DG]{
  \includegraphics[width=120mm]{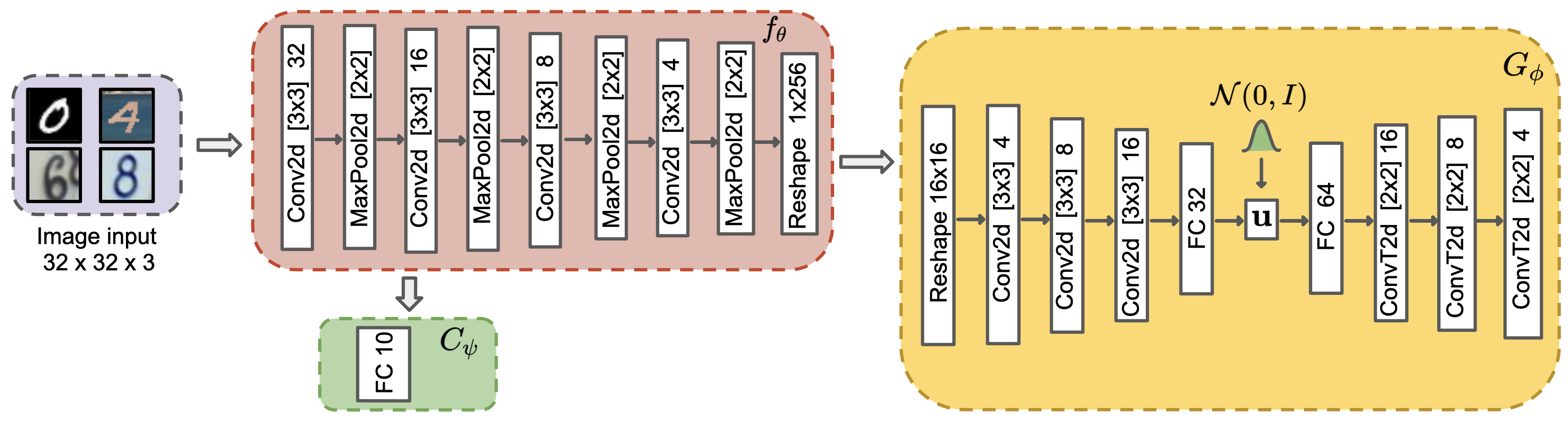}
}
\\
\vspace{0.8cm}
\subfloat[Multi-source DG on VLCS with GAN]{
  \includegraphics[width=120mm]{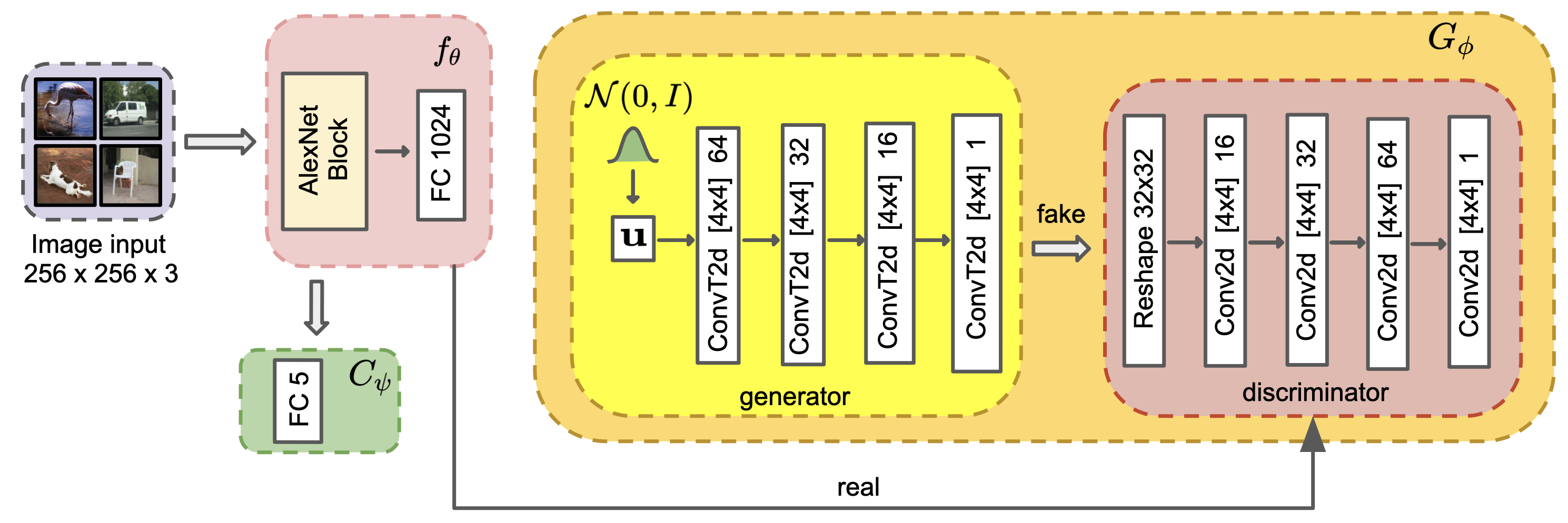}
}
\\
\vspace{0.8cm}
\subfloat[Multi-source DG on PACS with ResNet-50]{
  \includegraphics[width=120mm]{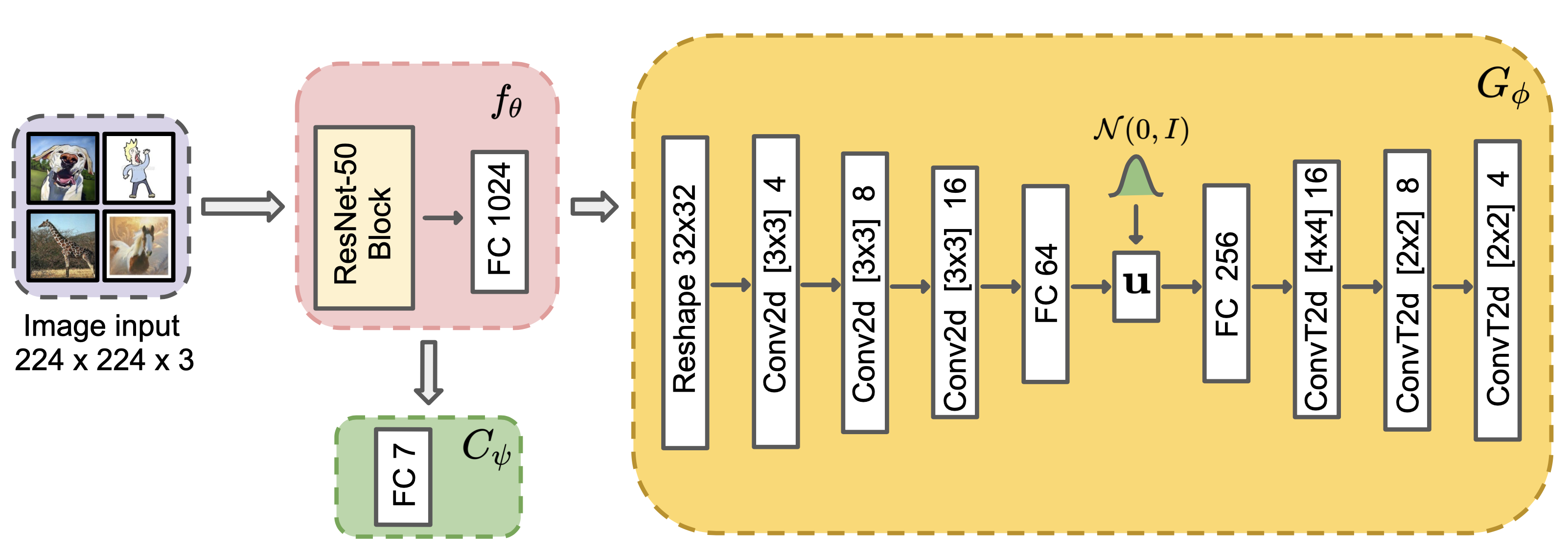}
}
\\
\vspace{0.8cm}
\subfloat[Robust DG on CIFAR-10-C with Wide Residual Network (WRN)]{
  \includegraphics[width=120mm]{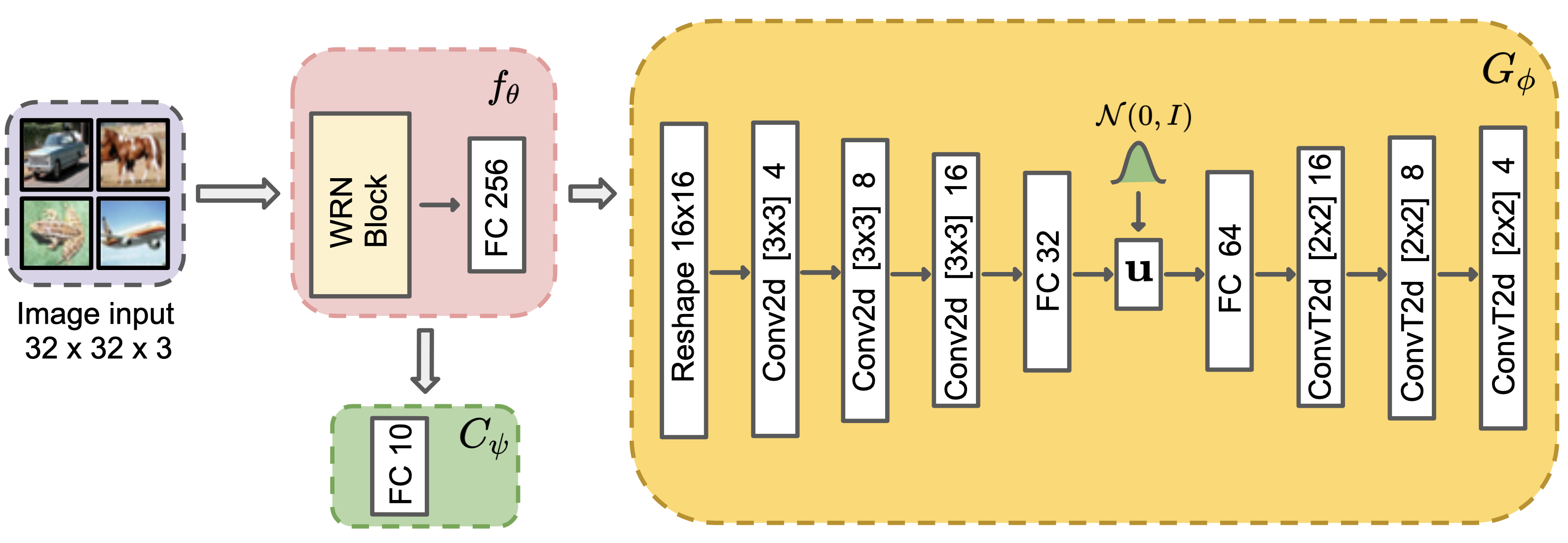}
}
\caption{
Architectures used for each dataset. Conv2d, ConvT2d, MaxPool2d and FC are Convolution 2D, Convolution Transpose 2D, 2D Max Pooling and Fully Connected layers, respectively. $\mathcal{C}_\psi$ is a single hidden layer classifier. Square bracket represents kernel size with number of output channels written right after it. $\mathbf{u}$ represents a vector from the latent space of $G_\phi$. ReLU/LeakyReLU activations are used in all blocks. All components $f_\theta$, $G_\phi$ and $\mathcal{C}_\psi$ are trained independently.}
\label{fig:architectures}
\end{figure*}

\section{Limitations}
One potential drawback of the proposed method is the requirement of a relatively larger inference time as optimization is performed for each target example.
 While this is a potential drawback of the method, we believe it doesn't prevent its practical use, as also noted in \cite{sun2020test}. Further, there are only a few hundred iterations needed (very less as compared to training) which takes about 50-100ms (on GPU) to execute. Also note that we do not update any of the model parameters during inference using the target data.

{\small
\bibliographystyle{ieee_fullname}
\bibliography{egbib}
}

\end{document}